\documentclass{article}

    \PassOptionsToPackage{numbers, compress}{natbib}

\usepackage[final]{neurips_2022}

\usepackage{microtype}
\usepackage{graphicx}
\usepackage{subfig}
\usepackage{booktabs} 
\usepackage[utf8]{inputenc} 
\usepackage[T1]{fontenc}    
\usepackage{url}            
\usepackage{amsfonts}       
\usepackage{nicefrac}       
\usepackage{amsthm}
\usepackage{lipsum}
\usepackage{wrapfig}
\usepackage{xcolor}
\usepackage{tabularx}
\usepackage[mathscr]{euscript}
\usepackage{multicol}
\usepackage{multirow}
\usepackage{hhline}
\usepackage{array}
\usepackage{framed}
\usepackage{algorithm,algorithmic}
\usepackage{amsmath}
\usepackage{amssymb}
\usepackage{bm}
\usepackage{color}
\newtheorem{lemma}{Lemma}
\newtheorem{definition}{Definition}
\newtheorem{prop}{Proposition}
\usepackage[colorlinks = true,
            linkcolor = blue,
            urlcolor  = blue,
            anchorcolor = blue]{hyperref}
\definecolor{COLOR}{rgb}{0, 0, 0}
    
\title{ComENet: Towards Complete and Efficient Message Passing for 3D Molecular Graphs}

\author{%
  Limei Wang\thanks{Equal contribution} \\
  Texas A\&M University\\
  College Station, TX 77843 \\
  \texttt{limei@tamu.edu} \\
   \And
   Yi Liu\footnotemark[1] \\
  Florida State University\\
  Tallahassee, FL 32306 \\
   \texttt{liuy@cs.fsu.edu} \\
   \AND
   Yuchao Lin \\
  Texas A\&M University\\
  College Station, TX 77843 \\
   \texttt{kruskallin@tamu.edu} \\
   \And
   Haoran Liu\\
  Texas A\&M University\\
  College Station, TX 77843 \\
   \texttt{liuhr99@tamu.edu} \\
   \AND
   Shuiwang Ji \\
  Texas A\&M University\\
  College Station, TX 77843 \\
   \texttt{sji@tamu.edu} \\
}

\begin{document}

\maketitle

\begin{abstract}
Many real-world data can be modeled as 3D graphs, 
but learning representations that incorporates 3D information completely and efficiently is challenging. Existing methods either use partial 3D information, 
or suffer from excessive computational cost.
To incorporate 3D information completely and efficiently,
we propose a novel message passing scheme that operates within 1-hop neighborhood.
Our method guarantees full completeness of 3D information on 3D graphs by achieving global
and local completeness.
Notably, we propose the important rotation angles 
to fulfill global completeness.
Additionally, we show that our method is orders of magnitude faster than prior methods.
We provide rigorous proof of completeness
and analysis of time complexity for our methods.
As molecules are in essence quantum systems, 
we
build the \underline{com}plete and \underline{e}fficient graph neural network (ComENet) by combing quantum inspired basis functions and the proposed message passing scheme.
Experimental results demonstrate the capability and efficiency of ComENet, especially on real-world
datasets that are large in both numbers and sizes of graphs.
Our code is publicly available
as part of the DIG library (\url{https://github.com/divelab/DIG}).
\end{abstract}

\section{Introduction}
In machine learning, structured objects such as molecules~\citep{duvenaud2015convolutional,wu2018moleculenet, wang2022advanced}, proteins~\citep{borgwardt2005protein,fout2017protein,liu2020deep,morehead2021geometric,jumper2021highly}, materials~\citep{xie2018crystal}, and quantum systems~\citep{kochkov2021learning, fu2022lattice} are usually modeled as graphs.
Original modeling shows basic connections between units and the resulted data type is known as 2D graphs. Accordingly, 2D graph neural networks (GNNs) have been intensively studied~\citep{duvenaud2015convolutional,gao2020topology,gao2018graph,xu2018powerful,zhang2018end,liu2021dig} 
and the message passing scheme~\citep{gilmer2017neural,battaglia2018relational,vignac2020building}
is shown effective to realize 2D GNNs.
However, in the modern machine learning era, it is increasingly accepted that modeling real-world data like molecules as 2D graphs leads to inborn defects for succeeding learning models.
In practice, 3D information is crucial, based on which some important geometries can be derived,
such as chemical bond lengths in 
molecular modeling. This essentially raises the need of a new data type, known as 3D graphs.

Formally, a 3D graph contains the original 2D graph as well as Cartesian coordinates for all nodes. In this work, we follow
invariant 3D GNNs~\citep{schutt2017schnet,klicpera_dimenet_2020,klicpera_dimenetpp_2020,shuaibi2021rotation,liu2022spherical,klicpera2021gemnet} 
that take relative 3D information,
like distances and angles, as inputs to networks.
Such relative geometries are naturally SE(3)-invariant.
The main challenge comes from what geometries should be computed such that 3D information is incorporated completely.
A most recent work SphereNet~\citep{liu2022spherical} 
shows that distance, angle, and torsion information is
necessary to incorporate more comprehensive 3D information.
However, SphereNet is only complete in local neighborhood,
failing to achieve global completeness and distinguish a wide range of molecular structures such
as conformers.
\textcolor{COLOR}{Let $n$ and $k$ denote the number of nodes and the average degree in a 3D graph.}
Existing methods also exhibit excessive time complexity of
$O(nk^2)$ or even $O(nk^3)$, severely preventing their scalability in real-world applications.

In this work, we propose ComENet as a \underline{com}plete and \underline{e}fficient graph neural network for 3D molecular graph learning.
We first formally provide the definition of \emph{completeness}.
Intuitively, a geometric transformation is considered as complete if it generates distinct representations for any two different 3D graphs.
Based on this, we propose a novel message passing scheme by faithfully fulfilling global completeness via the important rotation angles.
In addition, we design novel strategies to achieve local completeness,
largely reducing the computing complexity to $O(nk)$.
To elucidate the merits of the proposed methods,
we provide rigorous proof of geometric completeness achieved
by our method.
Combining the novel message passing scheme and quantum inspired features,
ComENet is developed for 3D molecular graphs.
We apply ComENet to two large-scale datasets including OC20 and  Molecule3D, and a 
commonly used dataset QM9.
Experiments show that ComENet performs similar to existing best methods,
but accelerates the training and inference by 6-10 times on various datasets.
\textbf{We summarize the contributions of ComENet as below}.
(i). To our best knowledge, it is the first rigorously complete pipeline for 3D molecular graph learning. Theoretically, it is guaranteed to incorporate 3D information completely without information loss. Practically, it can distinguish all molecular structures in nature.
(ii). It is highly efficient. The message passing is shown to be orders of magnitude faster than existing methods in terms of time complexity.
(iii). The great capability and efficiency of ComENet allow its scalability to real-world molecule datasets that are large in both numbers and sizes of graphs.
(iv). It achieves similar or better performance compared with existing methods, and dramatically accelerates the training and inference by 6-10 times.

\textbf{Relations with Prior Work.}
SphereNet~\citep{liu2022spherical} is a recent method that 
achieves local completeness with a complexity of $O(nk^2)$.
In summary, SphereNet is not complete and not efficient enough for processing large-scale molecular graphs.
However, ComENet is complete and much more efficient with a complexity of $O(nk)$.
Technically, a primary difference is that, ComENet
proposes to use the important rotation angles to achieve global
completeness.
As a result, ComENet is provably complete as shown in Sec.~\ref{sec:global}.
Practically, ComENet is able to identify a wide range of real-world
structures such as conformers, achieving completeness
at the conformer level,
as introduced in 
Sec.~\ref{sec:glob_conformer}.
In addition, ComENet and SphereNet both use $(d,\theta,\phi)$ in the spherical coordinate system (SCS) to obtain local completeness. 
Indeed, the fact that $(d,\theta,\phi)$ can specify the location of a point in SCS is widely known. The key difference lies in that, SphereNet
operates within 2-hop neighborhood, while ComENet operates within 1-hop neighborhood. When coupled with rotation angles, ComENet can achieve provable completeness with a reduced complexity of $O(nk)$.
This different message passing for local completeness in ComENet entails many differences with SphereNet, including building coordinate systems, defining $z$-axis, choosing reference nodes, and computing $(d,\theta,\phi)$,
as detailed in Sec.~\ref{sec:local}.
All the computing procedures for ComeNet are described in detail in Algorithm~\ref{alg:alg} of Appendix~\ref{sec:alg}.

\vspace{-5pt}
\section{The Proposed Message Passing Scheme}

\subsection{Notations \& Definitions} \label{sec:def}
We first formally define notations and the concept of \emph{completeness} used in this paper.

\textbf{Notations.}
A 3D graph $G$ can be represented as $G=(V, A, P)$.
The node feature matrix $V=[\textbf{v}_1, \textbf{v}_2,\cdot\cdot\cdot,\textbf{v}_n]^T \in \mathbb{R}^{n \times d_v}$
with each $\textbf{v}_i\in \mathbb{R}^{d_v}$.
The adjacency matrix $A\in \mathbb{R}^{n \times n}$,
based on which we additionally define there is an edge $e_{ij}$ if $A[i][j]=1$.
The position matrix $P=[\textbf{p}_1, \textbf{p}_2,\cdot\cdot\cdot,\textbf{p}_n]^T \in \mathbb{R}^{n \times 3}$,
where $\textbf{p}_i = (x_i, y_i, z_i)\in \mathbb{R}^{3}$ is the position vector for node $i$
given in the Cartesian coordinate system (CSC).
The relative position $\textbf{p}_{ij}$ of node $j$ to node $i$ is defined as
$\textbf{p}_{ij} = \textbf{p}_{j} - \textbf{p}_{i}$.
Particularly, throughout this paper, we define $k$ as the average degree
for $G$.

We then formally define \emph{completeness} given a geometric transformation $\mathcal{T}$.
In particular, we aim at incorporating 3D information in 3D molecular graphs. Hence,
our definition of \emph{completeness} is set from the geometric view.
Generally, $\mathcal{T}$ maps a 3D graph $G=(V,A,P)$ to a geometric representation with size $m\times h$\textcolor{COLOR}{, where $m$ is the number of transformed geometric features and $h$ is the feature size}. 
$\mathcal{T}$ can be different dependent on different methods, resulting in different $m \in\mathbb{N}^+$ or
$h \in\mathbb{N}^+$.
For example, SchNet~\citep{schutt2017schnet} \textcolor{COLOR}{only computes the distance for each edge based on the coordinates of the two nodes connected by this edge.}
Thus, SchNet maps $G$ to a representation with size $m\times h$,
where $m$ is the number of edges in $G$ and $h=1$.
However, such geometric transformation is not complete.
We provide the definition of \emph{completeness} as below:

\begin{definition}[Completeness]
For two 3D graphs $G_1=(V, A, P_1)$ and $G_2=(V, A, P_2)$,
a geometric transformation $\mathcal{T}:(\mathbb{R}^{n\times d_v}, \mathbb{R}^{n\times n}, \mathbb{R}^{n\times 3}) \mapsto \mathbb{R}^{m\times d}$ is considered as complete when
\begin{equation*}
\mathcal{T}(G_1) = \mathcal{T}(G_2) \iff \exists R \in \mathrm{SE(3)}, P_1 = R(P_2).
\end{equation*}
\label{def: gt}
\end{definition}
\vspace{-20pt}
\textcolor{COLOR}{Here $\mathrm{SE(3)}$ denotes the Special Euclidean group in 3 dimensions. It include all rotations and translations in 3D~\citep{adams2022learning, hoogeboom2022equivariant,liu2022spherical, satorras2021n}. Thus $R$ is a transformation that combines rotation and translation. A rotation transformation can be represented with a $3 \times 3$ rotation matrix, and a translation transformation can be represented by a $3 \times 1$ vector. Matrix form of $\mathrm{SE(3)}$ is provided in Appendix~\ref{SE(3)}.
The reason why we introduce $\mathrm{SE(3)}$ lies in that, a combination of rotation and translation will not change the 3D conformation of a 3D graph.
In Def.~\ref{def: gt}, if $P_1$ and $P_2$ are in the same $\mathrm{SE(3)}$ group, then $G_1$ and $G_2$ would share the same 3D conformation.
As a result, $G_1$ and $G_2$ would be the same 3D graph.}
Intuitively, a complete geometric transformation $\mathcal{T}$ can distinguish any two different 3D graphs.
This is to say, as long as two 3D graphs differ in \textcolor{COLOR}{3D conformations}, their outputs from $\mathcal{T}$ would be different.

\subsection{Global Completeness via Rotation Angles}   \label{sec:global}
Existing studies focus on the complete representation learning of local neighborhood.
Earlier methods like SchNet~\citep{schutt2017schnet} and DimeNet~\cite{klicpera_dimenet_2020} cannot \textcolor{COLOR}{achieve} local completeness.
In a more recent method SphereNet, completeness is guaranteed within \textcolor{COLOR}{edge-based 1-hop} local neighborhood,
but fails to hold in the whole 3D graph. 
In this section, we move a step forward to
formally fulfill global completeness  
for a given 3D molecular graph.
Particularly, for the purpose of clear illustration, we safely assume local completeness is already obtained by exiting methods like SphereNet.

\begin{wrapfigure}[14]{l}{0.64\textwidth}\vspace{-20 pt}
     \centering
     \subfloat[Illustration of a 2-hop local structure.]
     {\includegraphics[width=0.3\textwidth]{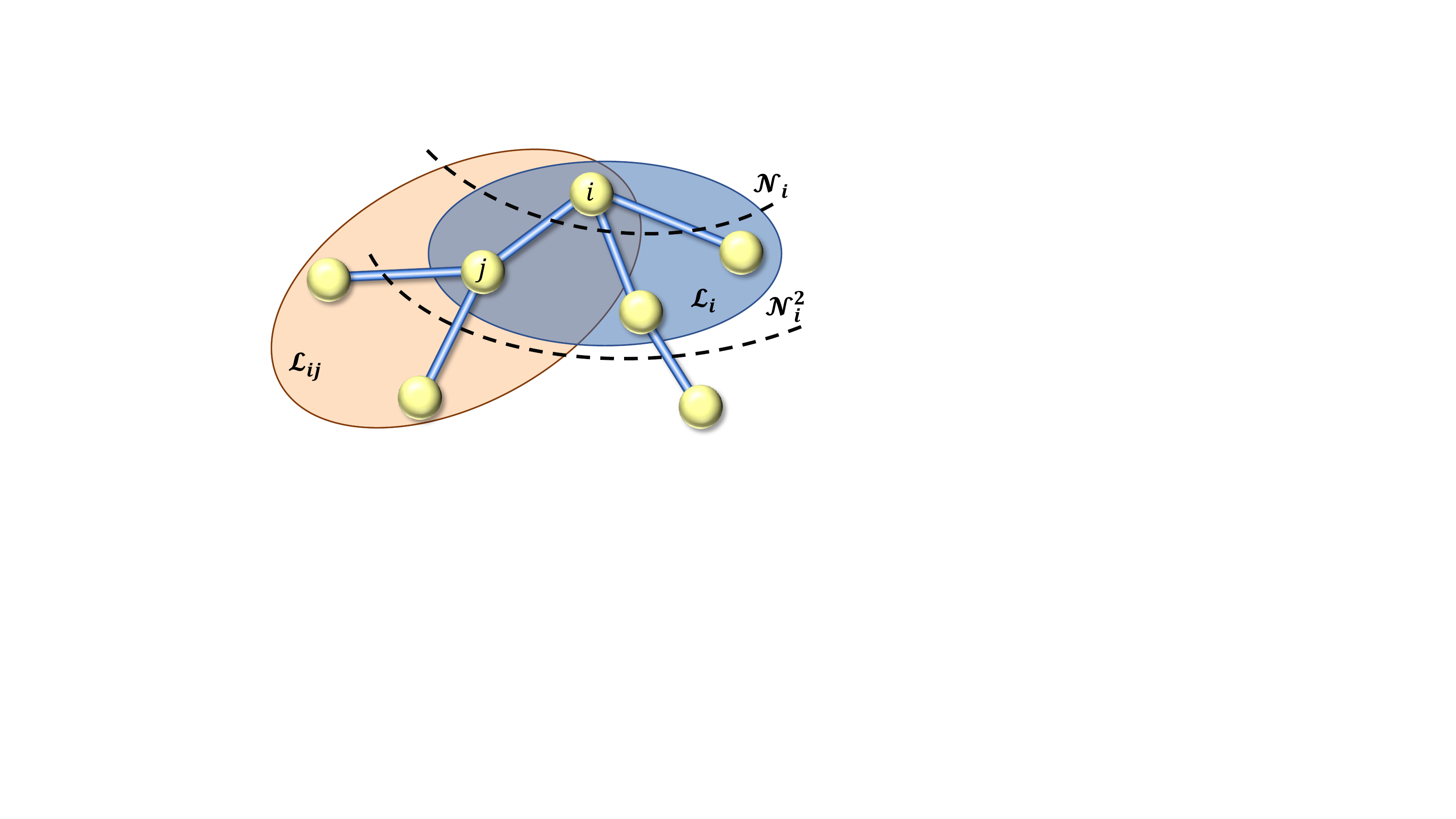}\label{fig:two-hop}}
     \quad
     \subfloat[Illustration of computing the rotation angle for an edge $e_{ij}$.]
     {\includegraphics[width=0.3\textwidth]{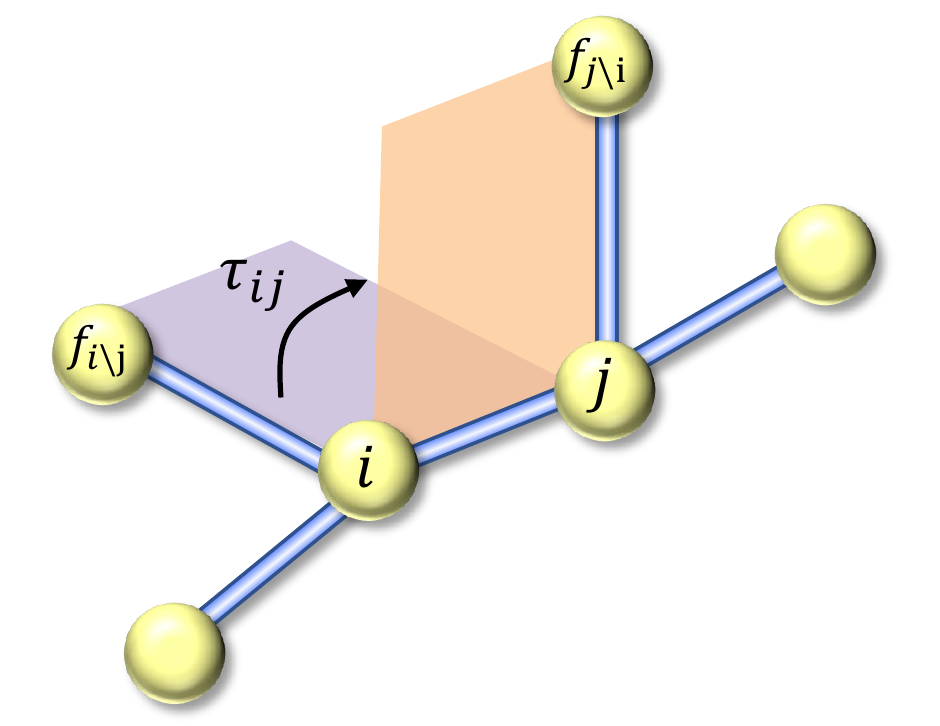}\label{fig:torsion}}
     \vspace{-6 pt}
    \caption{Illustrations of how to achieve global completeness in our proposed methods.}
    \label{fig:method}
    \vspace{-10 pt}
\end{wrapfigure}

Without loss of generality, we start by describing our method
to attain full completeness within 2-hop neighborhood, as illustrated in Fig.~\ref{fig:method}\subref{fig:two-hop}.
Formally, for a center node $i$, we let $\mathcal{N}_i$ and $\mathcal{N}^2_i$ denote two sets of indices of $i$'s 1-hop
and 2-hop neighboring nodes, respectively.
We also define any node $i$ and its 1-hop neighborhood as a local structure.
Then the whole 2-hop neighborhood of node $i$
can be viewed as $1+|\mathcal{N}_i|$ local structures centered in $i$ and $\mathcal{N}_i$, defined as 
$\mathcal{L}_{i}$ and $\mathcal{L}_{ij, j\in \mathcal{N}_i}$, respectively. 
\textcolor{COLOR}{As shown in Fig.~\ref{fig:method}\subref{fig:two-hop}, $\mathcal{L}_{i}$ is the local structure centered in $i$, and $\mathcal{L}_{ij}$ is the local structure centered in $j$.}
Apparently, each local structure $\mathcal{L}_{ij, j\in \mathcal{N}_i}$ shares the common edge $e_{ij}$ with the local structure $\mathcal{L}_{i}$.
Given the complete representation for each local structure, for the structure $\mathcal{L}_{i} \cup \mathcal{L}_{ij}$, the only remaining degree of freedom is the rotation angle of edge $e_{ij}$, denoted as $\tau_{ij}$.
With the rotation angles for all the $|\mathcal{N}_i|$ common edges specified, we can obtain a complete representation for 2-hop neighborhood of node $i$. 
Achieving completeness beyond 2-hop neighborhood is similar.
Overall, after considering rotation angles, the global completeness can be easily guaranteed when it gradually generalizes from $n$- to $(n+1)$-hop neighborhood.

Essentially, each edge in an input graph can be treated as a common edge between different local structures.
We reveal how to compute the rotation angle $\tau_{ij}$ for each edge $e_{ij}$ in Fig.~\ref{fig:method}\subref{fig:torsion}.
Specifically, we choose two reference nodes whose indices are $f_{i\backslash j}$ and $f_{j\backslash i}$ for node $i$ and $j$, respectively.
$f_{i\backslash j}$ denotes the index of $i$'s nearest neighboring node except $j$, and $f_{j\backslash i}$ denotes $j$'s nearest neighboring node except $i$.
Then $\tau_{ij}$ for edge $e_{ij}$ is the angle from the plane formed by $f_{i\backslash j},i,j$ to the plane formed by $i,j,f_{j\backslash i}$.
As analyzed previously, the global conformation of the input graph can be identified based on all local structures and rotation angles.
As a result, given fixed local structures, 
the global completeness the input 3D molecular graph is fulfilled by additionally considering the rotation angle of each edge, as introduced in this section.
\textcolor{COLOR}{Note that in terms of the selection of reference nodes for computing a rotation angle, we think a selection strategy is valid as long as it is applied to every edge consistently. In this work, for an edge $e_{ij}$, we choose the nearest neighboring nodes for $i$ and $j$ as reference nodes, which are $f_{i\backslash j}$ and $f_{j\backslash i}$, respectively. We apply this selection strategy to all edges in a 3D graph for computing corresponding rotation angles.
With such selection strategy, we can prove that our method is complete as shown in Sec.~\ref{sec:comp}.
We also show in Appendix~\ref{sec:proof} that it is easier to prove completeness using nearest neighboring nodes as reference nodes.
}

\subsection{Rotation Angles for Conformer Identification}  \label{sec:glob_conformer}
The proposed rotation angles 
in Sec.~\ref{sec:global} play a crucial role
in identifying some important molecular structures, such as conformers.
In nature, a real molecule exits as an ensemble of interconventing 3D structures,
known as conformers~\citep{axelrod2022geom,axelrod2020molecular,ganea2021geomol}.
Different conformers posses the same 2D molecular graph, but differ in 3D structures.
Generally, a conformer for a molecule exits with a certain probability and may exhibit 
distinct properties~\citep{axelrod2022geom,axelrod2020molecular}.
As shown in Fig.~\ref{fig:conformer}(a),
different conformers for the molecule butane ($\text{C}_4\text{H}_{10}$)
show varying conformation energy.
To this end, it is important to design complete 3D GNNs for 
identifying molecules at the conformer level.
From the geometry perspective of view,
a conformer distinguishes itself from others mainly
through varying rotation angles of chemical bonds~\citep{ganea2021geomol}.
As shown in Fig.~\ref{fig:conformer}(b), 
given the fixed ethyl (-$\text{C}_2\text{H}_{5}$) of both sides, 
the only degree of freedom is the rotation angle of the C-C bond.
In literature, the ethyl is formulated as a local 3D graph. 
Existing studies focus on the complete representation of 
such local structures, failing to identify the whole 3D graph globally.
Essentially, they can only distinguish different molecules, trying to achieve
completeness at the molecule level rather that the finer comformer level.
By integrating the rotation angles as in Sec.~\ref{sec:global} into the message passing scheme,
our methods can fulfill rigorous completeness at the conformer level and can distinguish all conformers in nature.

\begin{figure}[t]
    \centering
    \includegraphics[width=0.9\textwidth]{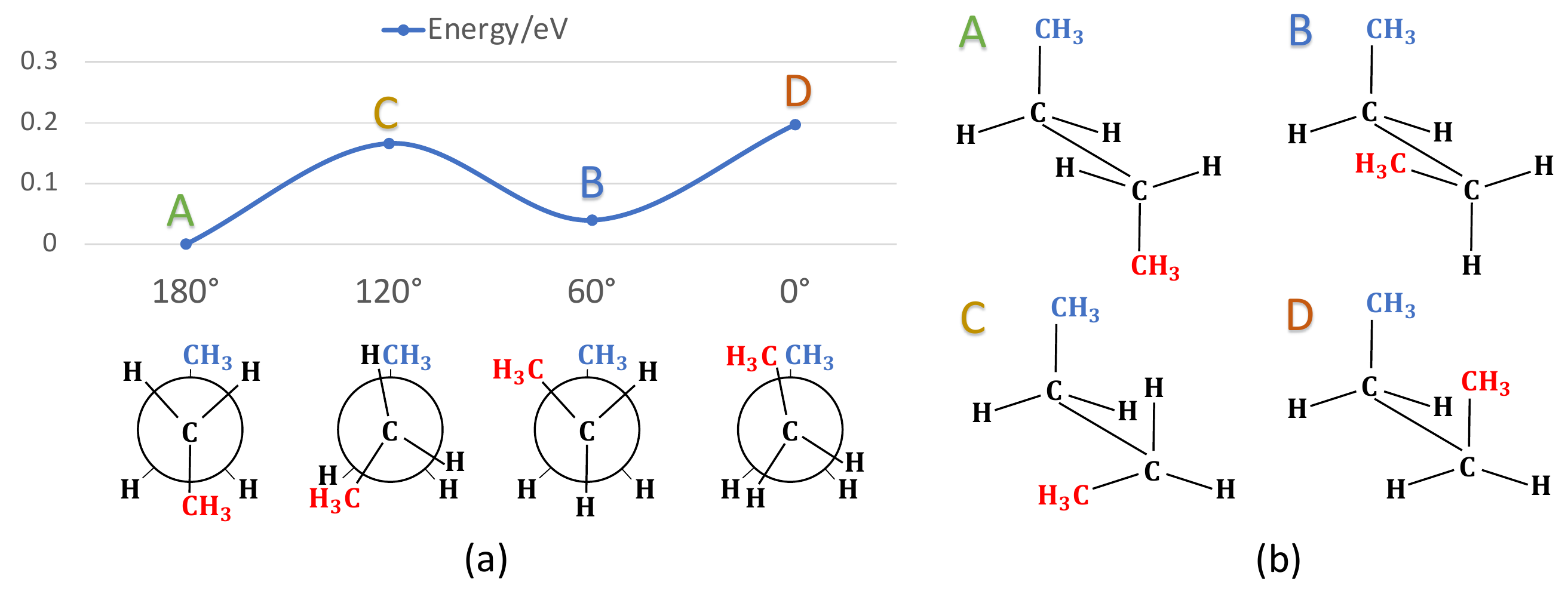}
    \vspace{-8 pt}
    \caption{(a). Illustration that the relative conformation energy of butane is a function of the rotation angle of the C-C bond.
    (b). A 3D view of the four conformers in (a).
    }\label{fig:conformer}
    \vspace{-15 pt}
\end{figure}

\subsection{Local Completeness with Improved Efficiency} \label{sec:local}

Formally, each node and its local neighborhood can be viewed as a local structure.
Existing studies focus on the learning of local structures and 
SphereNet achieves local completeness.
However, SphereNet induces the complexity of $O(nk^2)$,
restricting its scalability on large molecules in practice.
Here, we design a novel strategy to guarantee local completeness
with the computing cost of $O(nk)$.

Specifically, 
we follow SphereNet and perform on the spherical coordinate system.
It is commonly known that the location of each node can be completely determined using the tuple $(d,\theta,\phi)$ in SCS.
SphereNet employs the directional message passing (DMP) fashion that operates within 2-hop neighborhood.
It first updates messages over edges thus the center edge is
$z$-axis in SCS. For node $i$, the computing of the tuple $(d,\theta,\phi)$ 
involves 2-hop information.
However, we view 1-hop neighborhood as a local structure, which requires all strategies in our local completeness to be different from SphereNet,
including building local coordinate systems, defining z-axis, picking reference nodes, and computing $(d,\theta,\phi)$.
First, we build a light local coordinate system for any node $i$'s corresponding local structure.
Similarly, the center node $i$ serves as the origin.
Then $z$-axis is defined as the direction from $i$ to its nearest neighbor $f_i$,
and $xz$-plane is further formed by $z$-axis and $i$'s second nearest neighbor $s_i$.
Finally, the tuple $(d,\theta,\phi)$ is computed within 1-hop neighborhood with a complexity of $O(nk)$.
Particularly, we analyze efficiency versus model expressiveness in Sec.~\ref{sec:effi}. We show that compared with the DMP fashion used by SphereNet,
our method operating within 1-hop neighborhood
hurts the model expressiveness a bit by largely
improves the efficiency.

\subsection{Message Passing Scheme} \label{sec:mp}
Based upon global completeness achieved in Sec.~\ref{sec:global}
and improved local completeness introduced in Sec.~\ref{sec:local},
the complete geometric transformation $\mathcal{T}$ required by
Def.~\ref{def: gt} should be formulated based on a 4-tuple as 
$(d,\theta,\phi, \tau)$.
Specifically, we build such transformation within 1-hop neighborhood,
and a 4-tuple is computed for each edge.
Hence, for a 3D graph $G=(V,E,P)$,
the full expression for $\mathcal{T}$ is
$\mathcal{T}(G)=[(d_{ij},\theta_{ij},\phi_{ij}, \tau_{ij})]_{i=1,...,n; j\in \mathcal{N}_i} \in \mathbb{R}^{m\times 4}$,
where $m$ is the number of edges in $G$.
Especially, as $\mathcal{T}$ converts absolute Cartesian coordinates in $P$
to relative information, it is naturally SE(3)-invariant
as required in Def.~\ref{def: gt}.
To this end, we formally build our message passing scheme as
\begin{equation}
    \textbf{v}^{\prime}_i = g\left(\textbf{v}_i, \sum_{j\in\mathcal{N}_i}f\left(\textbf{v}_j, d_{ij}, \theta_{ij}, \phi_{ij}, \tau_{ij}\right)\right),
    \label{eq:mp}
\end{equation}
where $g$ and $f$ can be implemented by neural networks or mathematical operations.
Intuitively, our message passing is established in 1-hop local neighborhood and all edges connecting regions beyond.
Essentially, $d_{ij}$, $\theta_{ij}$, and $\phi_{ij}$
specify the 1-hop local neighborhood, and
$\tau_{ij}$ determines the orientation of the 
local neighborhood. By doing this, the complete representation 
for a whole 3D molecular graph is eventually achieved.
The formulas of computing of $d_{ij}$, $\theta_{ij}$, and $\phi_{ij}$, and $\tau_{ij}$
are shown in Algorithm~\ref{alg:alg} in Appendix~\ref{sec:alg},
along with detailed description of the
complexity of $O(nk)$.
\textcolor{COLOR}{
Overall, our formal analysis in Sec.~\ref{sec:global}, Sec.~\ref{sec:glob_conformer}, and Sec.~\ref{sec:local}
lead to the proposed message passing scheme defined in Eq.~\ref{eq:mp}.
It is the first fully complete scheme with great efficiency of $O(nk)$.
We also provide rigorous proof on completeness and analysis on efficiency of our message passing in Sec.~\ref{sec:merits}.
Note that to achieve efficiency, our message passing scheme adopts a novel strategy that computes all the needed geometries within 1-hop neighborhood.
Hence, it can not be directly applied to existing architectures built in 2-hop neighborhood, such as DimeNet++~\cite{klicpera_dimenetpp_2020} and SphereNet~\cite{liu2022spherical}.
To this end, we design a new network to implement the proposed message passing scheme, as detailed in Sec.~\ref{sec:network}.
}

\section{Merits of Our Methods} \label{sec:merits}

\subsection{Geometric Completeness}  \label{sec:comp}

\begin{prop}
For a strongly connected 3D graph $G=(V,E,P)$, its geometric transformation
$\mathcal{T}(G)=[(d_{ij},\theta_{ij},\phi_{ij}, \tau_{ij})]_{i=1,...,n; j\in \mathcal{N}_i}$
is complete.
\label{theorem1}
\end{prop}

\begin{proof}
\vspace{-5 pt}
We employ mathematical induction and assume the number of nodes in a 3D graph is $n$.
\textcolor{COLOR}{Note that the 3D graph we consider is strongly connected, which means that there exist a path between any two nodes in the graph. All the molecules in nature can be constructed as strongly connected graphs.}

Base case: It is obvious that the 3D structure of $G$ can be identified when $n=1,2$.
Hence, we let $n=3$ be the base case, where the completeness can be achieved by only considering $d$ and $\theta$ in $\mathcal{T}$.

Inductive hypothesis: The claim that $\mathcal{T}$ is complete holds for the node numbers of $n$ up to $k\geq3$.

Inductive step: Let $n=k+1$. Without loss of generality, among the existing $k$ nodes, we safely assume $i$ and 
its neighboring nodes $c_{k = 1,2,...}$ form the local region of interest.
Then $j$ is the index of the newly $(k+1)$-th node connected to the center node $i$.
To show global completeness of the whole graph, based on Def.~\ref{def: gt},
we only need to prove that the relative position of the new node $j$ is uniquely determined given
$\mathcal{T}$.
We propose the following lemma for this as

\begin{lemma}
Assume a strongly connected 3D graph $G=(V,E,P)$ with more than 2 nodes is fully identified.
If a new node $j$ is connected to a node $i$ of $G$ following the geometric transformation $\mathcal{T}(G)=[(d_{ij},\theta_{ij},\phi_{ij}, \tau_{ij})]_{i=1,...,n; j\in \mathcal{N}_i}$, then $\textbf{p}_{ij}$ is uniquely determined.
\label{lemma1}
\vspace{-5 pt}
\end{lemma}
The proof of Lemma~\ref{lemma1} is provided in Appendix~\ref{sec:proof}.
With Lemma~\ref{lemma1} successfully proved, we show that such geometric transformation $\mathcal{T}$
can determine a unique 3D graph.
Hence, the $\emph{if}$ condition in Def.~\ref{def: gt} holds.
In addition, as $\mathcal{T}$ renders purely relative 3D information like distance and angle, it's naturally 
SE(3) invariant. 
Hence, the $\emph{only if}$ condition in Def.~\ref{def: gt} holds.
Overall, based on Def.~\ref{def: gt}, we complete the proof of Prop.~\ref{theorem1}.
\vspace{-5 pt}
\end{proof}

Intuitively, since a molecular graph is strongly connected, two
arbitrary nodes are connected by at least one path. Hence, starting from the existing structure, we can restore the relative position of any new node step by step along a path with finite length.
As a result, geometric completeness for the whole 3D graph with any number of nodes can be guaranteed in our message passing scheme.
Theoretically, 3D information of the input 3D molecular graph
is fully captured without information loss.
In practice, our method can distinguish
all structures in nature.

\subsection{Efficiency}  \label{sec:effi}
\textbf{Efficiency versus Model Expressiveness:} Our method induces the complexity of $O(nk)$ by operating within 1-hop neighborhood.
Existing methods, such as like DimeNet and SphereNet, employ the DMP fashion that update edges
within 2-hop neighborhood, inducing the complexity of $O(nk^2)$.
Notably, DMP~\citep{stokes2020deep,yang2019analyzing,klicpera_dimenet_2020, liu2022spherical} incorporates
2-hop information in one single layer but $(n+1)$-hop when stacking $n$ layers.
However,
stacking $n$ proposed message passing layers is already able to incorporate information from $n$-hop neighborhood.
Obviously, compared with a network containing several DMP layers, a network with the same number of our
proposed message passing layers merely hurts the model expressiveness a bit, but
significantly improves the model efficiency.

\textbf{Efficiency via Less Torsion Angles:}
In addition, our method achieves completeness by
computing $O(nk)$ torsion angles, which is efficient especially 
compared with methods including~\citet{klicpera2021gemnet,adams2022learning,ganea2021geomol}.
Given a scenario where a global region is the union of two local regions with $n_1$ and $n_2$ nodes.
As introduced in Sec.~\ref{sec:global} and Sec.~\ref{sec:local}, our method computes the same number of torsion angles as nodes in a local region,
then employs a rotation angle (torsion angle essentially). Hence, the number of torsion angles that our method computes is $n_1+n_2+1$.
\citet{ganea2021geomol} computes a torsion angle
based on one pair of nodes, each of which is from a separate local region.
Hence, the number of torsion angles needed is $n_1\times n_2$.
Apparently, our method reduces the number of torsion angles from $O(nk^3)$ to $O(nk)$,
which is significant considering the computing of torsion is excessively expensive.

\section{ComENet} \label{sec:network}

\begin{wrapfigure}[22]{r}{0.6\textwidth}\vspace{-25 pt}
    \centering
    \includegraphics[width=0.6\textwidth]{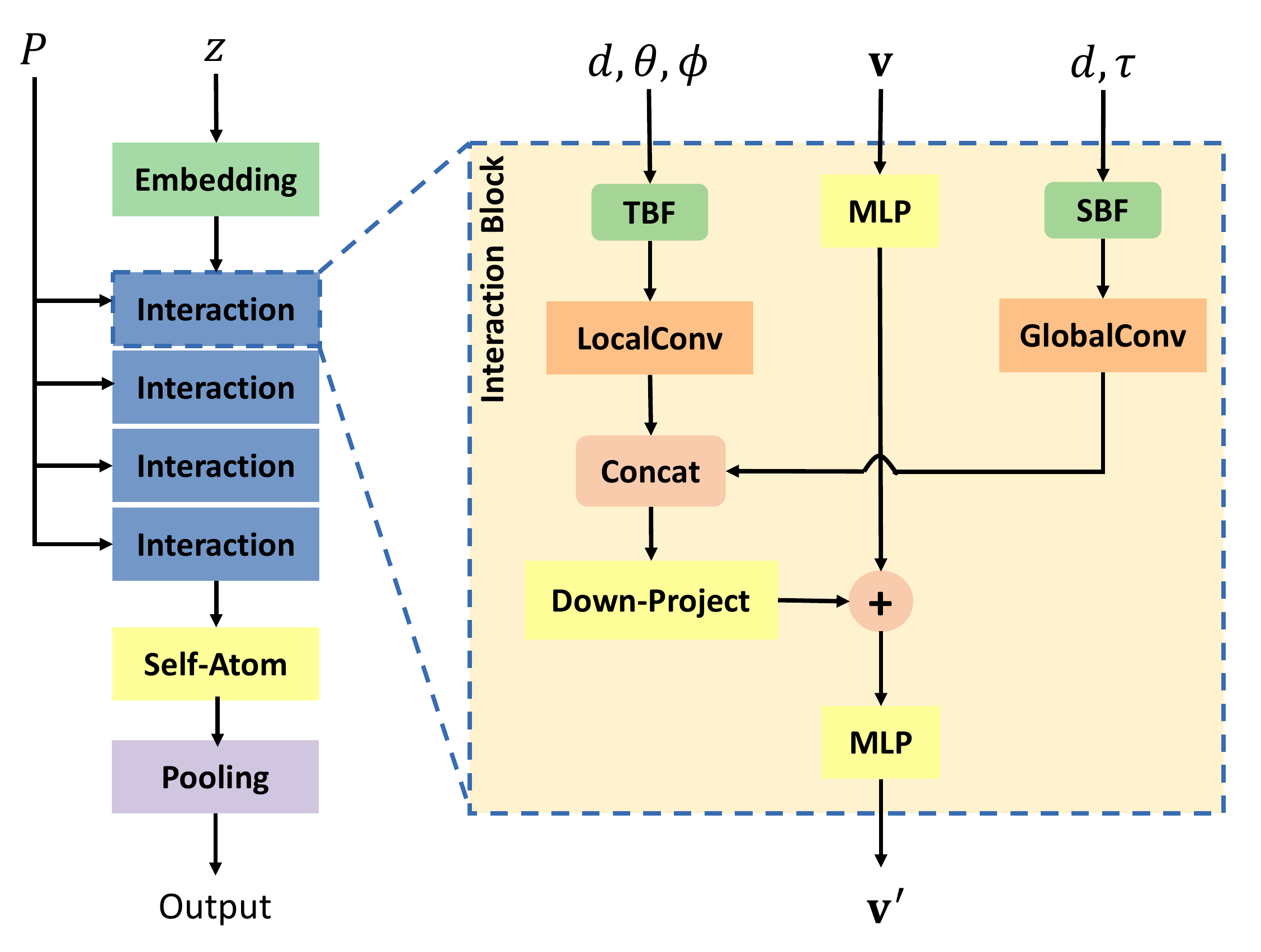}
    \vspace{-18 pt}
    \caption{Illustration of ComENet with an overview (left) and the interaction layer (right).
    TBF and SBF denote the basis functions for tuples $(d,\theta,\phi)$ and $(d,\tau)$.
    LocalConv and GlobalConv denote the proposed local and global convolution layers.
    Concat is the concatenation operation and Down-Project is a linear layer to reduce feature dimensions.
    + denotes the element-wise sum operation.
    }\label{fig:model}
    \vspace{-20 pt}
\end{wrapfigure}

Based upon the message passing scheme introduced in Sec.~\ref{sec:mp}, 
we propose the complete and efficient graph neural network (ComENet) as shown in Fig.~\ref{fig:model}. 
\textcolor{COLOR}{Existing invariant 3D GNN methods~\cite{schmidhuber2015deep,klicpera_dimenet_2020,klicpera_dimenetpp_2020,liu2022spherical,klicpera2021gemnet,schutt2021equivariant}
share the similar architecture pipeline, which contains an input block, several interaction blocks, and an output block.
Our ComENet also follows such architecture fashion along with several novel components, such as self-atom layers and specifically designed local and global graph convolution layers, to better fulfill our proposed message passing scheme in Eq.~\ref{eq:mp}.}
Generally, ComENet consists of an embedding layer, multiple interaction layers, a self-atom layer, and a pooling layer. To be in line with the message passing scheme in Sec.~\ref{sec:mp},
we take the updating process for node $i$'s feature vector $\textbf{v}_i$
as an example to describe the network. \textcolor{COLOR}{In practice, feature vectors for all the nodes in a graph are updated simultaneously.}
In particular, we omit all indices in Fig.~\ref{fig:model} and below for clear presentation.
Specifically, \textit{embedding Layer} converts atom type z to an initial node feature vector \textcolor{COLOR}{$\textbf{v}$} via learnable atom type embeddings~\cite{schutt2017schnet,klicpera_dimenet_2020}.
\textit{Interaction Layer} updates node feature vector $\textbf{v}$ based on features of the neighboring nodes and geometric features $(d,\theta,\phi,\tau)$ in Eq.~\ref{eq:mp} using the local and global graph convolution layers. A detailed description of interaction layer is provided in Appendix~\ref{sec:model architecture}.
\textit{Self-atom layer} is used to update each node feature and project the feature dimension into 1. 
And \textit{pooling layer} is a sum-pooling performing on all node features to obtain final predictions.

\section{Related Work}

We consider how to represent 3D information in 3D molecular graphs~\citep{atz2021geometric,bronstein2021geometric}.
One category of methods are \textit{equivariant 3D GNNs} that directly use coordinates in the CSC as inputs to  networks~\citep{thomas2018tensor,anderson2019cormorant,fuchs2020se,schutt2021equivariant,batzner2021se}. 
These methods are efficient but suffer from several setbacks.
Firstly, each network component needs to be carefully designed to be rotation equivariant of input graphs. 
Secondly, the reason why this category of methods are efficient lies in that they only use the type-1 basis in Spherical harmonics, which is an approximation essentially. It is proved in ~\citep{thomas2018tensor} that theoretically, $l$ is infinite in terms of type-$l$ basis. In practice, $l$ should be at least 2 for achieving satisfactory performance. Type-1 basis essentially coerces the conv kernel to be in a narrow learning space, which imposes a hard constraint on the network capability.
However, when using type-2 basis, the conv kernel can be more expressive while the efficiency issue
would emerge as a new bottleneck.
Thirdly, the performance of such equivariant GNNs is
shown to be worse than invariant 3D GNNs~\citep{liu2022spherical}.

In this work, we follow \textit{invariant 3D GNNs} to inherit the merit of SE(3)-invariance by investigating
relative 3D information, which is used in both representation learning tasks
~\citep{schutt2017schnet,klicpera_dimenet_2020,klicpera_dimenetpp_2020,shuaibi2021rotation,liu2022spherical,klicpera2021gemnet}
and coordinates generation tasks~\citep{ganea2021geomol,Simm2020Reinforcement,Simm2021SymmetryAware,xu2021molecule3d,jumper2021highly,baek2021accurate}.
Since the relative information is SE(3)-invariant of input 3D graphs, the employed networks favorably achieve the invariance merit.
We focus on the 3D graph learning problem and existing methods either capture partial 3D information or suffer from high computational cost.
For example, SchNet~\cite{schutt2017schnet} only considers distance information and  DimeNet~\cite{klicpera_dimenet_2020} further incorporates angles between bonds. They both integrate incomplete 3D information that the network capacity is limited in practice.
SphereNet~\cite{liu2022spherical} generates approximate complete 3D representations by using
distance, angle, and torsion information but the complexity is $O(nk^2)$.
GemNet~\cite{klicpera2021gemnet} is based on quadruplets of nodes, which is more expensive. Our objective is to build a fully complete 3D graph net with a much lower computational budget.

There also exist some methods in literature using \textit{both absolute and relative 3D information}~\cite{godwin2021very,hu2021forcenet,ying2021do}. 
In this work, we only use relative information as input to avoid inherent limitations of equivariant 3D GNNs.
Moreover, we design a novel message passing such that the computational cost
is comparable with that of equivariant 3D GNNs.

\begin{wraptable}[7]{r}{0.56\textwidth}\vspace{-10 pt}
\begin{center}
\vspace{-10pt}
\caption{Statistics of the datasets.}
\vspace{-5pt}
\label{tb:datasets}
\resizebox{0.56\textwidth}{!}
{\begin{tabular}{l|ccc}\toprule
Dataset &OC20 &Molecule3D &QM9 \\
\midrule
\# Graphs &660,010 &3,899,647 &130,831 \\
Split Type &Pre-defined &Random/Scaffold &Random \\
Split Ratio &70:15:15 &6:2:2 &84:8:8 \\
\# Nodes/Graph &77.75 &29.11 &18.02 \\
\bottomrule
\end{tabular}}
\end{center}
\vspace{-10pt}
\end{wraptable}

\section{Experiments}  \label{sec:exp}

We examine the power and efficiency
of ComENet on two large-sacle datasets including 
Open Catalyst 2020 (OC20)~\cite{chanussot2021open} and
Molecule3D~\cite{xu2021molecule3d}, 
and the mostly commonly used datastet QM9~\cite{ramakrishnan2014quantum}.
The statistics of three datasets are provided in Table~\ref{tb:datasets}. Detailed descriptions of the datasets are provided in Appendix~\ref{sec:data}. 
In particular, the Molecule3D contains about 4 million 3D molecular graphs,
and the OC20 has the average graph size of 77.75.
Baseline methods include GIN-Virtual~\cite{hu2021ogb}, 
CGCNN~\cite{xie2018crystal}, 
SchNet~\cite{schutt2017schnet}, 
PhysNet~\cite{unke2019physnet}, 
MGCN~\cite{lu2019molecular}, 
DimeNet~\cite{klicpera_dimenet_2020}, 
DimeNet++~\cite{klicpera_dimenetpp_2020}, 
SphereNet~\cite{liu2022spherical}, 
PaiNN~\cite{schutt2021equivariant}, 
GemNet~\cite{klicpera2021gemnet}.
Unless otherwise specified, the values for baseline methods are taken from the referred papers.
For the ComENet, we use data loader in the PyTorch Geometric library~\cite{Fey/Lenssen/2019} to load the datasets.
All the models are trained using the Adam optimizer~\cite{kingma2014adam} and the optimal hyperparameters
are obtained on validation sets using grid search. 
Experimental setup and search space for all
models are provided in Appendix~\ref{sec:setup}. 
Code is integrated in the DIG library~\cite{liu2021dig} and available at~\url{https://github.com/divelab/DIG}.

\begin{table*}[t]
    \begin{center}
        \caption{Results on IS2RE including computing cost in training\&inference
        and performance in terms of
        energy MAE and the percentage of EwT of the ground truth energy. 
        Training time is the average time per epoch during training using 1 GPU.
        Performance is reported for models trained on the All training dataset.
        The best performance is shown in bold and the second best is shown with underlines.
        }
    \label{tb:result_oc20}
    \resizebox{\textwidth}{!}
    {\begin{tabular}{l cc | ccccc | ccccc  }
    \toprule
    & \multicolumn{2}{c|}{Time } &\multicolumn{5}{c|}{Energy MAE [eV] $\downarrow$} & \multicolumn{5}{c}{EwT $\uparrow$}  \\
    \cmidrule(l{4pt}r{4pt}){2-3}
    \cmidrule(l{4pt}r{4pt}){4-8}
    \cmidrule(l{4pt}r{4pt}){9-13}
 Model &Train &Infer. & ID &  OOD Ads & OOD Cat & OOD Both &Average& ID &  OOD Ads & OOD Cat & OOD Both &Average\\
\midrule
CGCNN &18min &1min &0.6203 &0.7426 &0.6001 &0.6708 &0.6585 &3.36\% &2.11\% &3.53\% &2.29\% &2.82\% \\
SchNet &10min &1min &0.6465 &0.7074 &0.6475 &0.6626 &0.6660 &2.96\% &2.22\% &3.03\% &2.38\% &2.65\% \\
DimeNet++ &230min &4min &0.5636 &0.7127 &0.5612 &0.6492 &0.6217 &4.25\% &2.48\% &4.40\% &2.56\% &3.42\% \\
GemNet-T &200min &4min &\underline{0.5561} &0.7342 &0.5659 &0.6964 &0.6382 &\underline{4.51\%} &2.24\% &4.37\% &2.38\% &3.38\% \\
SphereNet &290min &5min &0.5632 &\underline{0.6682} &\underline{0.5590} &\underline{0.6190} &\underline{0.6023} &\textbf{4.56\%} &\underline{2.70\%} &\textbf{4.59\%} &\underline{2.70\%} &\textbf{3.64\%} \\
ComENet &20min &1min &\textbf{0.5558} &\textbf{0.6602} &\textbf{0.5491} &\textbf{0.5901} &\textbf{0.5888} &4.17\% &\textbf{2.71\%} &\underline{4.53\%} &\textbf{2.83\%} &\underline{3.56\%} \\
\bottomrule
\end{tabular}}
\end{center}
\end{table*}

\subsection{OC20} \label{sec:exp_oc20}
The Open Catalyst 2020 (OC20) dataset is a newly released large-scale dataset with millions of DFT relaxations to model and discover catalysts.
In this work, we focus on Initial Structure to Relaxed Energy (IS2RE) task, which is the most common task in catalyst discovery.
Descriptions of the data and tasks are provided in Appendix~\ref{sec:data}. 
The ground truth of the test set is not publicly available, 
therefore, we compare the results of different methods on the validation set. 
The evaluation metrics include the energy MAE and the percentage of Energies within a Threshold (EwT) of the ground truth energy. 
The values for the baseline methods are taken from~\citet{chanussot2021correction, liu2022spherical}.
Notably, we aim to predict relaxed energy directly from initial structure and do not compare with some methods using relaxation~\cite{chanussot2021open}, trajectory information, or relaxed structures.
Using relaxation~\citep{chanussot2021open,klicpera2021gemnet,shuaibi2021rotation} is computationally expensive during prediction 
while the relaxation trajectory and relaxed structures~\citep{godwin2021very,ying2021do} are hard to obtain in practice.

Table~\ref{tb:result_oc20} shows that ComENet outperforms all the baseline methods in terms of energy MAE, 
which is also used as the main evaluation metric in the Open Catalyst Challenge~\cite{occhallenge}. 
ComENet achieves best performance on two splits and the second best on the other two splits
in terms of EwT.
Specifically, ComENet reduces the average energy MAE by 0.0135, which is 2.2\% of the second best model. 
Note that ComENet achieves the best results on the OOD Both split in terms of both energy MAE and EwT.
In practice, it is common that test data is in the different domain with the training data.
Hence, OOD Both can test the generalization capability of learning models.
More importantly, ComENet is much more efficient than methods like DimeNet++ and SphereNet. 
For example, SphereNet needs 5 hours per epoch while ComENet only requires 20 minutes using the same computing infrastructure (NVIDIA RTX A6000 48GB). Overall, compared with existing best methods,
ComENet achieves better performance and largely reduces training by at least 10 times.

\begin{wraptable}[]{r}{0.6\textwidth}
\vspace{-15 pt}
\begin{center}
\caption{Comparisons between ComENet and other models in terms of computing cost and HOMO-LUMO gap MAE on Molecule3D for both random and scaffold splits. Train time is the average training time per epoch.}
\label{tb:result_molecule3d}
\small
{\begin{tabular}{lcc|cc}\toprule
&\multicolumn{2}{c|}{Time} 
&\multicolumn{2}{c}{MAE}  \\
\cmidrule(l{4pt}r{4pt}){2-3}
\cmidrule(l{4pt}r{4pt}){4-5}
Model  &Train &Inference &Random &Scaffold \\
\midrule
GIN-Virtual &15min &2min &0.1036  &0.2371 \\
SchNet &14min &3min &0.0428 &0.1511 \\
DimeNet++ &133min &16min &0.0306 &0.1214 \\
SphereNet &182min &28min &0.0301  &0.1182 \\
ComENet &22min &3min &0.0326 &0.1273 \\
\bottomrule
\end{tabular}}
\end{center}
\vspace{- 8 pt}
\end{wraptable}

\subsection{Molecule3D} \label{sec:exp_molecule3d}

The Molecule3D dataset~\cite{xu2021molecule3d} is a newly proposed large-scale dataset, including around 4 million molecules with precise ground-state 3D information derived from DFT and molecular properties. 
We focus on the prediction of the HOMO-LUMO gap as it is one of the most practically-relevant quantum chemical properties of molecules. A detailed description of the Molecule3D dataset is provided in Appendix~\ref{sec:data}.
As this is a newly proposed dataset, we run baseline methods including GIN-Virtual~\cite{hu2021ogb}, SchNet~\cite{schutt2017schnet}, DimeNet++\cite{klicpera_dimenetpp_2020} and SphereNet~\cite{liu2022spherical},
among which GIN-Virtual is a powerful baseline for 2D graphs while the others are for 3D graphs. 
All the models are trained using the same computing infrastructure (Nvidia GeForce RTX 2080 Ti 11GB).

\begin{wrapfigure}[]{l}{.45\textwidth}
    \centering
    \includegraphics[width=.45\textwidth]{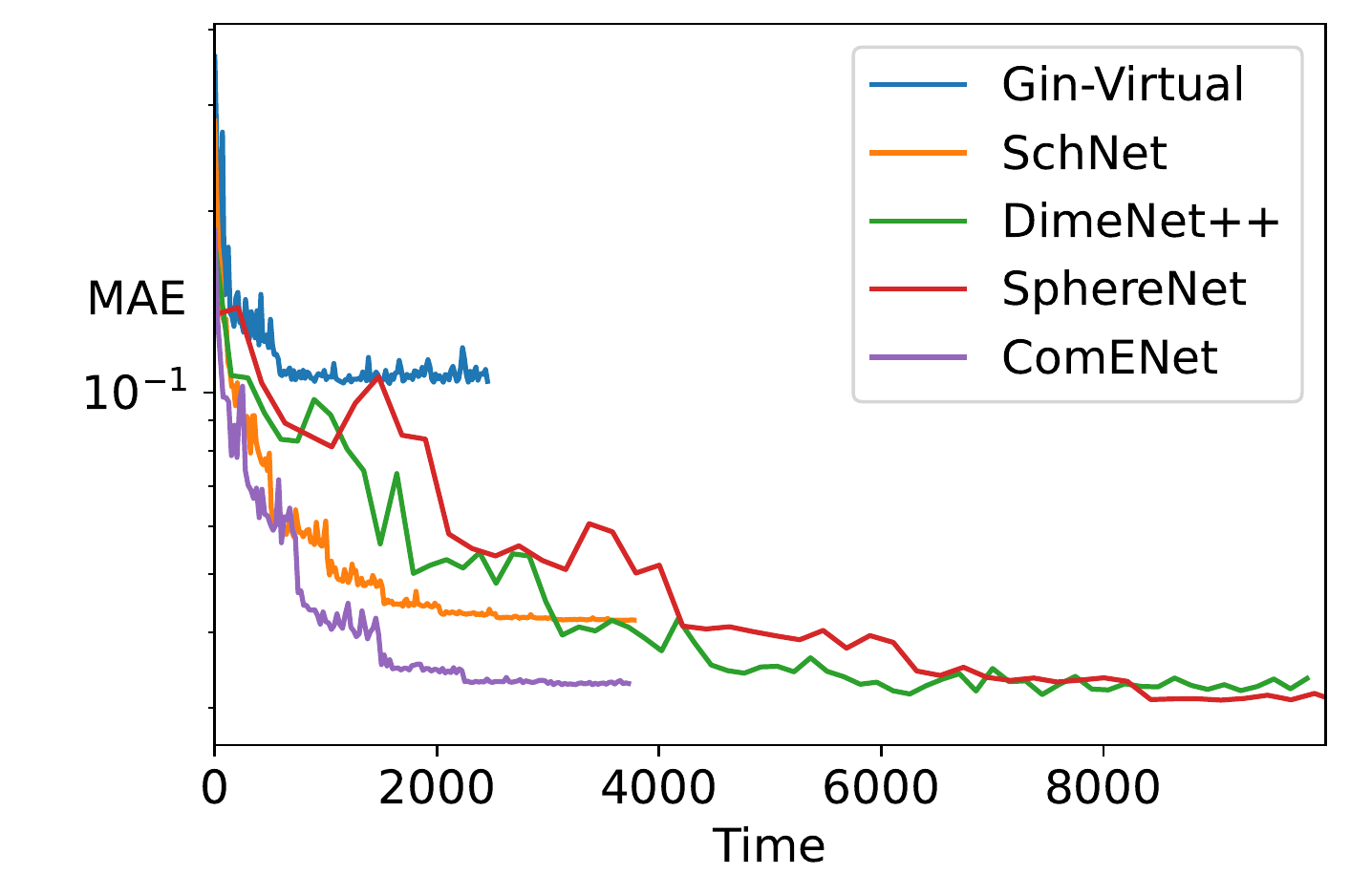}
    \vspace{-15 pt}
    \caption{Total training time for different methods on Molecule3D.
    }\label{fig:time}
    \vspace{-15 pt}
\end{wrapfigure}

Table~\ref{tb:result_molecule3d} shows that
ComENet dramatically reduces training time by 6-9 times compared with DimeNet++ and SphereNet, 
and costs similar time as GIN-virtual and SchNet that only considers distance information. 
In terms of performance, our ComENet performs much better than SchNet for both random and scaffold splits with similar time and computing costs, 
but a little worse than DimeNet++ and SphereNet. 
This may be due to the molecules are relatively small in Molecule3D compared with OC20, 
the structures that our complete strategy can distinguish may not exist in the dataset. 
However, considering the comparable performance and high efficiency,
our ComENet is more practically useful than other methods on such large datasets.
In addition, Fig.~\ref{fig:time} also shows ComENet either converges much faster in terms of total training time or performs much better compared with other baselines.

\subsection{QM9} \label{sec:exp_qm9}

\begin{table*}[t]
\begin{center}
\caption{Comparisons between ComENet and other models
    in terms of MAE and the overall mean std. MAE on QM9.
}\label{tb:result_qm9}
\resizebox{\textwidth}{!}
{\begin{tabular}{llcccccccc}
\toprule
Property &Unit &SchNet &PhysNet &MGCN &DimeNet &DimeNet++ &PaiNN &SphereNet &ComENet \\
\midrule
$\mu$                  &D      &0.033 &0.0529 &0.0560 &0.0286 &0.0297 &0.012 &0.0245 & 0.0245\\
$\alpha$            &${a_0}^3$ &0.235 &0.0615 &0.0300 &0.0469 &0.0435 &0.045 &0.0449 & 0.0452\\
$\epsilon_\text{HOMO}$ &meV    &41    &32.9   &42.1   &27.8   &24.6   &27.6   &22.8  & 23.1\\
$\epsilon_\text{LUMO}$ &meV    &34    &24.7   &57.4   &19.7   &19.5   &20.4   &18.9  & 19.8\\
$\Delta\epsilon$       &meV    &63    &42.5   &64.2   &34.8   &32.6   &45.7   &31.1  & 32.4\\
$\left< R^2 \right>$ &${a_0}^2$ &0.073 &0.765  &0.110  &0.331  &0.331  &0.066  &0.268 & 0.259\\
ZPVE  &meV    &1.7   &1.39   &1.12   &1.29   &1.21   &1.28   &1.12  & 1.20\\
$U_0$    &meV    &14    &8.15   &12.9   &8.02   &6.32   &5.85   &6.26  & 6.59\\
$U$     &meV    &19    &8.34   &14.4   &7.89   &6.28   &5.83   &6.36  & 6.82\\
$H$     &meV    &14    &8.42   &14.6   &8.11   &6.53   &5.98   &6.33  & 6.86\\
$G$     &meV    &14    &9.4    &16.2   &8.98   &7.56   &7.35   &7.78  & 7.98\\
$c_\text{v}$ &$\frac{\mbox{cal}}{\mbox{mol K}}$ &0.033 &0.028  &0.038  &0.025  &0.023  &0.024  &0.022 & 0.024\\
\midrule
std. MAE &\%  &1.76  &1.37   &1.86   &1.05   &0.98   &1.01   &0.91  &0.93 \\
\bottomrule
\end{tabular}}
\end{center}
\vspace{-10 pt}
\end{table*}

The QM9 dataset is a widely used dataset for predicting various properties of molecules. 
The evaluation metrics include the MAE for each property and the overall mean standardized MAE (std. MAE) for all the 12 properties. A detailed description of the dataset is provided in Appendix~\ref{sec:data}.
Notably, we do not list results for PPGN~\cite{maron2019provably} and Cormorant~\cite{anderson2019cormorant} since they use different train/val/test sizes. 
Table~\ref{tb:result_qm9} shows that ComENet 
is much better than the methods operating in 1-hop neighborhood like SchNet, PhyNet, and MGCN.
Compared with DimeNet, DimeNet++, PaiNN, and SphereNet,
ComENet achieves similar results on all properties
and the overall std. MAE.
Consistently, compared with methods operating in 2-hop neighborhood
like DimeNet, DimeNet++, and SphereNet, ComENet is much more efficient.

\subsection{Ablation Study for Identifying Conformers}

\begin{wrapfigure}[10]{r}{0.61\textwidth}
\vspace{-17 pt}
     \centering
     {\includegraphics[width=0.3\textwidth]{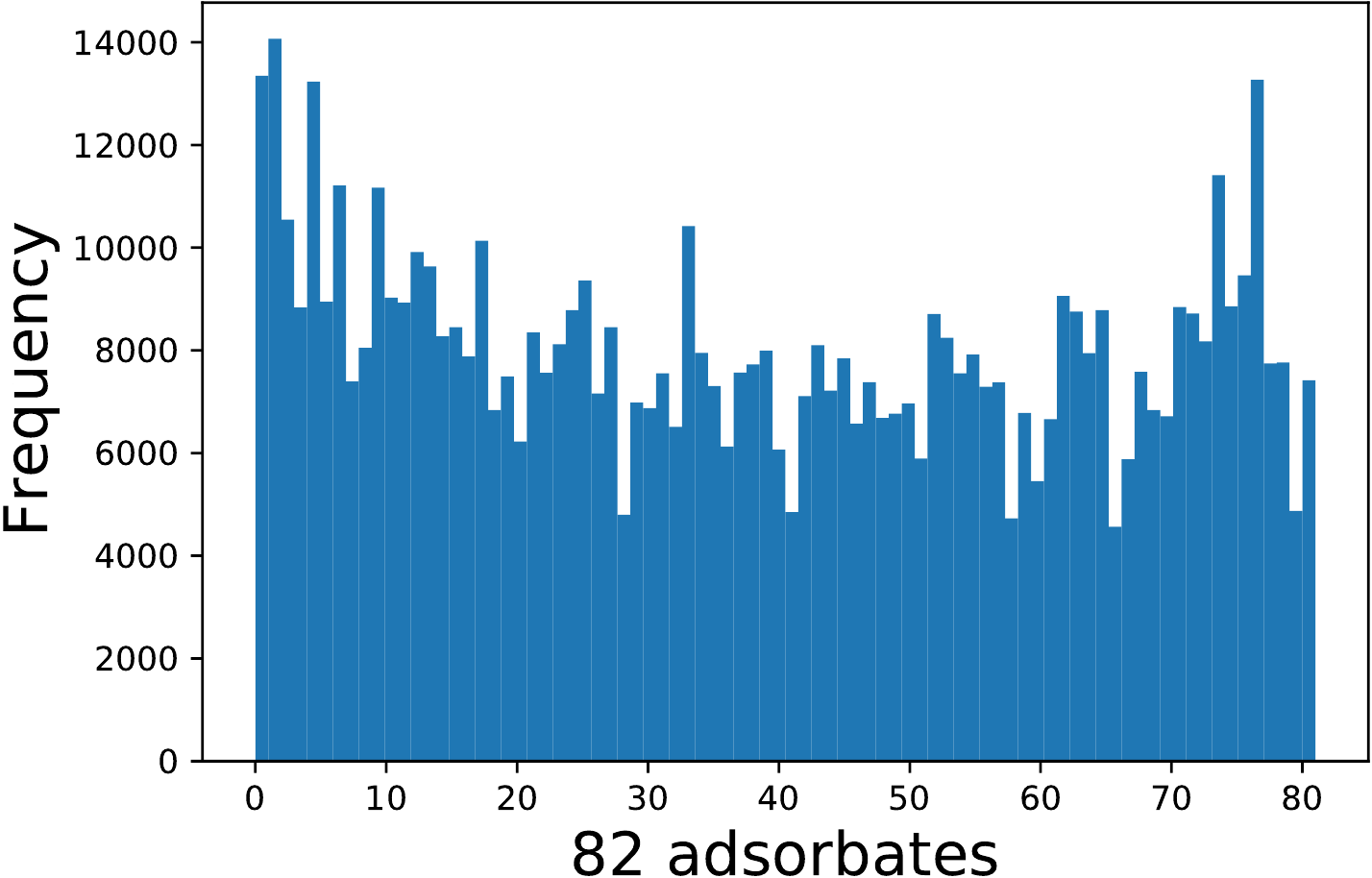}\label{fig:ads}}
     {\includegraphics[width=0.3\textwidth]{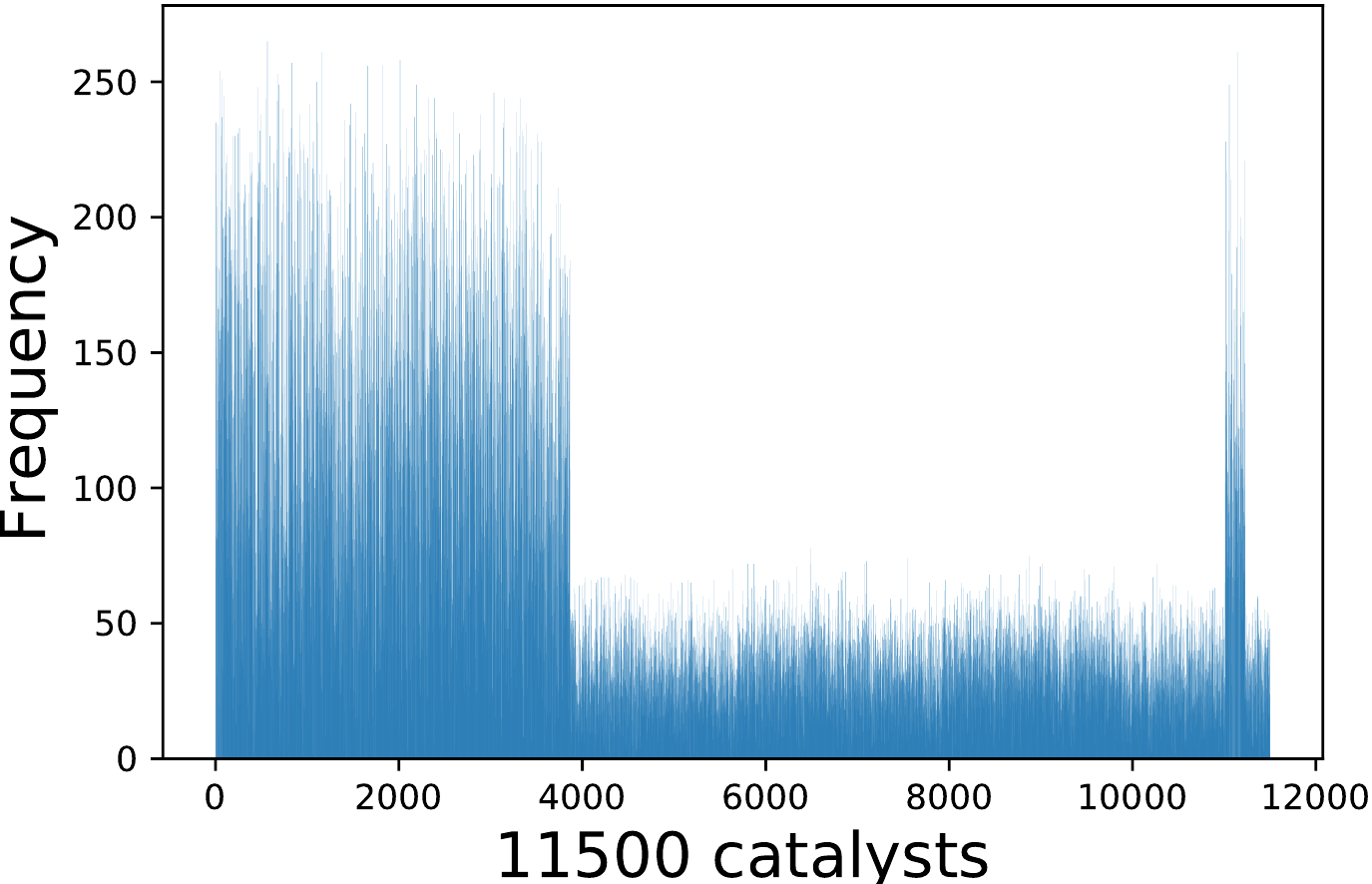}\label{fig:bulk}}
     \vspace{-12 pt}
    \caption{Distributions of adsorbates and catalysts in OC20.
    For y-axis, frequency counts the number of conformers for each individual adsorbate and catalyst.}
    \label{fig:oc_analysis}
    \vspace{-10 pt}
\end{wrapfigure}

\textcolor{COLOR}{
As mentioned in Sec.~\ref{sec:glob_conformer}, rotation angles are the main difference between different conformers~\citep{ganea2021geomol, jing2022torsional}. We investigate the contribution of our proposed rotation angles $\tau$ to demonstrate the effectiveness of our global complete representations. We conduct experiments on the OC20 dataset since there exist different conformers for molecules in this dataset. Specifically, there are 660,010 data samples in the OC20 dataset (IS2RE), where each sample is a combination of two parts, namely, adsorbate and catalyst.
There are 82 adsorbates and 11,500 catalysts used in the datasets. Each adsorbate inevitably corresponds to different conformers in the dataset, and it is similar to catalyst.
We show the number of conformers for each adsorbate and each catalyst in Fig.~\ref{fig:oc_analysis}.
We remove the rotation angle $\tau$ from ComENet and denote it as "ComENet w/o $\tau$". The results in Table~\ref{tb:ablation_oc20} show that removing rotation angles $\tau$ can harm the performance of ComENet, demonstrating the effectiveness of our global complete representations for identifying conformers. 
}

\begin{table*}[ht]
\vspace{-5 pt}
    \begin{center}
        \caption{Comparisons between ComENet and the model without rotation angles $\tau$ on OC20.
        }
        \vspace{-5 pt}
    \label{tb:ablation_oc20}
    \resizebox{\textwidth}{!}
    {\color{COLOR}\begin{tabular}{l | ccccc | ccccc  }
    \toprule
    &\multicolumn{5}{c|}{Energy MAE [eV] $\downarrow$} & \multicolumn{5}{c}{EwT $\uparrow$}  \\
    \cmidrule(l{4pt}r{4pt}){2-6}
    \cmidrule(l{4pt}r{4pt}){7-11}
 Model & ID &  OOD Ads & OOD Cat & OOD Both &Average& ID &  OOD Ads & OOD Cat & OOD Both &Average\\
\midrule
ComENet &\textbf{0.5558} &\textbf{0.6602} &\textbf{0.5491} &\textbf{0.5901} &\textbf{0.5888} &\textbf{4.17\%} &\textbf{2.71\%} & \textbf{4.53\%} &\textbf{2.83\%} & \textbf{3.56\%} \\
ComENet w/o $\tau$ & 0.5585 & 0.6851 & 0.5574 & 0.6186 & 0.6049 & 4.13\% & 2.65\% & 4.13\% & 2.75\% & 3.42\% \\
\bottomrule
\end{tabular}}
\end{center}
\vspace{-10 pt}
\end{table*}

\section{Conclusions, Limitations, Outlook, and Societal Impacts}
3D information is crucial for 3D molecular graph learning.
Existing methods either learn partial 3D information or induce high time complexity.
We propose ComENet that is both complete in incorporating 3D information and efficient 
with time complexity of $O(nk)$.
Particularly, we propose the novel rotation angles to fulfill global completeness.
ComENet can generalize to large-scale datasets,
accelerating training and inference by 6-10 times with superior or comparable performance.
\textcolor{COLOR}{Even though ComENet is the first complete and the most efficient 3D GNN model, there exists one major limitation, which is not only for
ComENet but also for existing 3D GNNs. Basically, existing 3D GNN models
are centered on the fact that 3D information is given in data.
However, acquiring 3D information itself is difficult and expensive in practice.
Current methods rely on experiments or
DFT-based computing, which is extremely time-consuming. 
It would be significant that machine
learning models can be developed to tackle this problem.
Looking forward, we can derive two directions for fulfilling such objective.
Firstly, we can study to generate 3D graphs either from 2D graphs or from scratch by developing generative models,
such as VAE, flow and diffusion models.
Especially, in some real-world applications like drug discovery, 2D molecules are usually not given.
This raises the need of developing new generation methods from scratch.
Secondly, we can target at a research case where we have a minimal set of training data with
3D information, but the vast unseen data or new data lack such 3D information.
We may develop novel contrastive learning components to force correspondence and consistency 
between 2D graphs and their 3D geometric views,
then integrate such components into end-to-end learning systems for
application based on 2D graph data.
ComENet can facilitate a plethora of important real-world applications, such as drug discovery and material discovery.
It can be used in several research domains including quantum chemistry and physics, material sciences, molecular dynamics simulations, etc.
Any negative societal impact associated with those applications and domains can be applied to our method.}

\begin{ack}
This work was supported in part by National Science Foundation grant IIS-1908220 and National Institutes of Health grant U01AG070112.
\end{ack}

\bibliography{deep}

\begin{thebibliography}{61}
\providecommand{\natexlab}[1]{#1}
\providecommand{\url}[1]{\texttt{#1}}
\expandafter\ifx\csname urlstyle\endcsname\relax
  \providecommand{\doi}[1]{doi: #1}\else
  \providecommand{\doi}{doi: \begingroup \urlstyle{rm}\Url}\fi

\bibitem[Adams et~al.(2022)Adams, Pattanaik, and Coley]{adams2022learning}
Keir Adams, Lagnajit Pattanaik, and Connor~W. Coley.
\newblock Learning {3D} representations of molecular chirality with invariance
  to bond rotations.
\newblock In \emph{International Conference on Learning Representations}, 2022.
\newblock URL \url{https://openreview.net/forum?id=hm2tNDdgaFK}.

\bibitem[Anderson et~al.(2019)Anderson, Hy, and Kondor]{anderson2019cormorant}
Brandon Anderson, Truong-Son Hy, and Risi Kondor.
\newblock Cormorant: Covariant molecular neural networks.
\newblock In \emph{Proceedings of the 33st International Conference on Neural
  Information Processing Systems}, pages 14537--14546, 2019.

\bibitem[Atz et~al.(2021)Atz, Grisoni, and Schneider]{atz2021geometric}
Kenneth Atz, Francesca Grisoni, and Gisbert Schneider.
\newblock Geometric deep learning on molecular representations.
\newblock \emph{Nature Machine Intelligence}, pages 1--10, 2021.

\bibitem[Axelrod and Gomez-Bombarelli(2020)]{axelrod2020molecular}
Simon Axelrod and Rafael Gomez-Bombarelli.
\newblock Molecular machine learning with conformer ensembles.
\newblock \emph{arXiv preprint arXiv:2012.08452}, 2020.

\bibitem[Axelrod and Gomez-Bombarelli(2022)]{axelrod2022geom}
Simon Axelrod and Rafael Gomez-Bombarelli.
\newblock {GEOM}, energy-annotated molecular conformations for property
  prediction and molecular generation.
\newblock \emph{Scientific Data}, 9\penalty0 (1):\penalty0 1--14, 2022.

\bibitem[Baek et~al.(2021)Baek, DiMaio, Anishchenko, Dauparas, Ovchinnikov,
  Lee, Wang, Cong, Kinch, Schaeffer, et~al.]{baek2021accurate}
Minkyung Baek, Frank DiMaio, Ivan Anishchenko, Justas Dauparas, Sergey
  Ovchinnikov, Gyu~Rie Lee, Jue Wang, Qian Cong, Lisa~N Kinch, R~Dustin
  Schaeffer, et~al.
\newblock Accurate prediction of protein structures and interactions using a
  three-track neural network.
\newblock \emph{Science}, 373\penalty0 (6557):\penalty0 871--876, 2021.

\bibitem[Battaglia et~al.(2018)Battaglia, Hamrick, Bapst, Sanchez-Gonzalez,
  Zambaldi, Malinowski, Tacchetti, Raposo, Santoro, Faulkner,
  et~al.]{battaglia2018relational}
Peter~W Battaglia, Jessica~B Hamrick, Victor Bapst, Alvaro Sanchez-Gonzalez,
  Vinicius Zambaldi, Mateusz Malinowski, Andrea Tacchetti, David Raposo, Adam
  Santoro, Ryan Faulkner, et~al.
\newblock Relational inductive biases, deep learning, and graph networks.
\newblock \emph{arXiv preprint arXiv:1806.01261}, 2018.

\bibitem[Batzner et~al.(2022)Batzner, Musaelian, Sun, Geiger, Mailoa,
  Kornbluth, Molinari, Smidt, and Kozinsky]{batzner2021se}
Simon Batzner, Albert Musaelian, Lixin Sun, Mario Geiger, Jonathan~P Mailoa,
  Mordechai Kornbluth, Nicola Molinari, Tess~E Smidt, and Boris Kozinsky.
\newblock E(3)-equivariant graph neural networks for data-efficient and
  accurate interatomic potentials.
\newblock \emph{Nature communications}, 13\penalty0 (1):\penalty0 1--11, 2022.

\bibitem[Borgwardt et~al.(2005)Borgwardt, Ong, Sch{\"o}nauer, Vishwanathan,
  Smola, and Kriegel]{borgwardt2005protein}
Karsten~M Borgwardt, Cheng~Soon Ong, Stefan Sch{\"o}nauer, SVN Vishwanathan,
  Alex~J Smola, and Hans-Peter Kriegel.
\newblock Protein function prediction via graph kernels.
\newblock \emph{Bioinformatics}, 21\penalty0 (suppl\_1):\penalty0 i47--i56,
  2005.

\bibitem[Bronstein et~al.(2021)Bronstein, Bruna, Cohen, and
  Veli{\v{c}}kovi{\'c}]{bronstein2021geometric}
Michael~M Bronstein, Joan Bruna, Taco Cohen, and Petar Veli{\v{c}}kovi{\'c}.
\newblock Geometric deep learning: Grids, groups, graphs, geodesics, and
  gauges.
\newblock \emph{arXiv preprint arXiv:2104.13478}, 2021.

\bibitem[Chanussot et~al.(2021{\natexlab{a}})Chanussot, Das, Goyal, Lavril,
  Shuaibi, Riviere, Tran, Heras-Domingo, Ho, Hu,
  et~al.]{chanussot2021correction}
Lowik Chanussot, Abhishek Das, Siddharth Goyal, Thibaut Lavril, Muhammed
  Shuaibi, Morgane Riviere, Kevin Tran, Javier Heras-Domingo, Caleb Ho, Weihua
  Hu, et~al.
\newblock Correction to “the open catalyst 2020 (oc20) dataset and community
  challenges”.
\newblock \emph{ACS Catalysis}, 11\penalty0 (21):\penalty0 13062--13065,
  2021{\natexlab{a}}.

\bibitem[Chanussot et~al.(2021{\natexlab{b}})Chanussot, Das, Goyal, Lavril,
  Shuaibi, Riviere, Tran, Heras-Domingo, Ho, Hu, et~al.]{chanussot2021open}
Lowik Chanussot, Abhishek Das, Siddharth Goyal, Thibaut Lavril, Muhammed
  Shuaibi, Morgane Riviere, Kevin Tran, Javier Heras-Domingo, Caleb Ho, Weihua
  Hu, et~al.
\newblock Open catalyst 2020 ({OC20}) dataset and community challenges.
\newblock \emph{ACS Catalysis}, 11\penalty0 (10):\penalty0 6059--6072,
  2021{\natexlab{b}}.

\bibitem[Chanussot et~al.(2021{\natexlab{c}})Chanussot, Das, Goyal, Lavril,
  Shuaibi, Riviere, Tran, et~al.]{occhallenge}
Lowik Chanussot, Abhishek Das, Siddharth Goyal, Thibaut Lavril, Muhammed
  Shuaibi, Morgane Riviere, Kevin Tran, et~al.
\newblock {Open Catalyst Challenge}.
\newblock \url{https://opencatalystproject.org/challenge.html},
  2021{\natexlab{c}}.

\bibitem[Duvenaud et~al.(2015)Duvenaud, Maclaurin, Iparraguirre, Bombarell,
  Hirzel, Aspuru-Guzik, and Adams]{duvenaud2015convolutional}
David~K Duvenaud, Dougal Maclaurin, Jorge Iparraguirre, Rafael Bombarell,
  Timothy Hirzel, Al{\'a}n Aspuru-Guzik, and Ryan~P Adams.
\newblock Convolutional networks on graphs for learning molecular fingerprints.
\newblock \emph{Advances in Neural Information Processing Systems},
  28:\penalty0 2224--2232, 2015.

\bibitem[Fey and Lenssen(2019)]{Fey/Lenssen/2019}
Matthias Fey and Jan~E. Lenssen.
\newblock Fast graph representation learning with {PyTorch Geometric}.
\newblock In \emph{ICLR Workshop on Representation Learning on Graphs and
  Manifolds}, 2019.

\bibitem[Fout et~al.(2017)Fout, Byrd, Shariat, and Ben-Hur]{fout2017protein}
Alex Fout, Jonathon Byrd, Basir Shariat, and Asa Ben-Hur.
\newblock Protein interface prediction using graph convolutional networks.
\newblock In \emph{Proceedings of the 31st International Conference on Neural
  Information Processing Systems}, pages 6533--6542, 2017.

\bibitem[Fu et~al.(2022)Fu, Zhang, Zhang, Ling, Xu, and Ji]{fu2022lattice}
Cong Fu, Xuan Zhang, Huixin Zhang, Hongyi Ling, Shenglong Xu, and Shuiwang Ji.
\newblock Lattice convolutional networks for learning ground states of quantum
  many-body systems.
\newblock \emph{arXiv preprint arXiv:2206.07370}, 2022.

\bibitem[Fuchs et~al.(2020)Fuchs, Worrall, Fischer, and Welling]{fuchs2020se}
Fabian Fuchs, Daniel Worrall, Volker Fischer, and Max Welling.
\newblock {SE}(3)-{T}ransformers: {3D} roto-translation equivariant attention
  networks.
\newblock \emph{Advances in Neural Information Processing Systems}, 33, 2020.

\bibitem[Ganea et~al.(2021)Ganea, Pattanaik, Coley, Barzilay, Jensen, Green,
  and Jaakkola]{ganea2021geomol}
Octavian-Eugen Ganea, Lagnajit Pattanaik, Connor~W Coley, Regina Barzilay,
  Klavs~F Jensen, William~H Green, and Tommi~S Jaakkola.
\newblock {GeoMol}: Torsional geometric generation of molecular {3D} conformer
  ensembles.
\newblock \emph{Advances in Neural Information Processing Systems}, 2021.

\bibitem[Gao and Ji(2019)]{gao2018graph}
Hongyang Gao and Shuiwang Ji.
\newblock Graph {U-Nets}.
\newblock In \emph{International Conference on Machine Learning}, pages
  2083--2092. PMLR, 2019.

\bibitem[Gao et~al.(2021)Gao, Liu, and Ji]{gao2020topology}
Hongyang Gao, Yi~Liu, and Shuiwang Ji.
\newblock Topology-aware graph pooling networks.
\newblock \emph{IEEE Transactions on Pattern Analysis and Machine
  Intelligence}, 43\penalty0 (12):\penalty0 4512--4518, 2021.

\bibitem[Gasteiger et~al.(2020)Gasteiger, Groß, and
  Günnemann]{klicpera_dimenet_2020}
Johannes Gasteiger, Janek Groß, and Stephan Günnemann.
\newblock Directional message passing for molecular graphs.
\newblock In \emph{International Conference on Learning Representations}, 2020.
\newblock URL \url{https://openreview.net/forum?id=B1eWbxStPH}.

\bibitem[Gasteiger et~al.(2021)Gasteiger, Becker, and
  G{\"u}nnemann]{klicpera2021gemnet}
Johannes Gasteiger, Florian Becker, and Stephan G{\"u}nnemann.
\newblock Gemnet: Universal directional graph neural networks for molecules.
\newblock \emph{Advances in Neural Information Processing Systems},
  34:\penalty0 6790--6802, 2021.

\bibitem[Gilmer et~al.(2017)Gilmer, Schoenholz, Riley, Vinyals, and
  Dahl]{gilmer2017neural}
Justin Gilmer, Samuel~S Schoenholz, Patrick~F Riley, Oriol Vinyals, and
  George~E Dahl.
\newblock Neural message passing for quantum chemistry.
\newblock In \emph{Proceedings of the 34th International Conference on Machine
  Learning-Volume 70}, pages 1263--1272. JMLR. org, 2017.

\bibitem[Godwin et~al.(2022)Godwin, Schaarschmidt, Gaunt, Sanchez-Gonzalez,
  Rubanova, Veli{\v{c}}kovi{\'c}, Kirkpatrick, and Battaglia]{godwin2021very}
Jonathan Godwin, Michael Schaarschmidt, Alexander~L Gaunt, Alvaro
  Sanchez-Gonzalez, Yulia Rubanova, Petar Veli{\v{c}}kovi{\'c}, James
  Kirkpatrick, and Peter Battaglia.
\newblock Simple {GNN} regularisation for {3D} molecular property prediction
  and beyond.
\newblock In \emph{International Conference on Learning Representations}, 2022.
\newblock URL \url{https://openreview.net/forum?id=1wVvweK3oIb}.

\bibitem[Hoogeboom et~al.(2022)Hoogeboom, Satorras, Vignac, and
  Welling]{hoogeboom2022equivariant}
Emiel Hoogeboom, Victor~Garcia Satorras, Cl{\'e}ment Vignac, and Max Welling.
\newblock Equivariant diffusion for molecule generation in {3D}.
\newblock In \emph{International Conference on Machine Learning}, pages
  8867--8887. PMLR, 2022.

\bibitem[Hu et~al.(2021{\natexlab{a}})Hu, Fey, Ren, Nakata, Dong, and
  Leskovec]{hu2021ogb}
Weihua Hu, Matthias Fey, Hongyu Ren, Maho Nakata, Yuxiao Dong, and Jure
  Leskovec.
\newblock {OGB}-{LSC}: A large-scale challenge for machine learning on graphs.
\newblock In \emph{Thirty-fifth Conference on Neural Information Processing
  Systems Datasets and Benchmarks Track (Round 2)}, 2021{\natexlab{a}}.
\newblock URL \url{https://openreview.net/forum?id=qkcLxoC52kL}.

\bibitem[Hu et~al.(2021{\natexlab{b}})Hu, Shuaibi, Das, Goyal, Sriram,
  Leskovec, Parikh, and Zitnick]{hu2021forcenet}
Weihua Hu, Muhammed Shuaibi, Abhishek Das, Siddharth Goyal, Anuroop Sriram,
  Jure Leskovec, Devi Parikh, and C~Lawrence Zitnick.
\newblock {ForceNet}: A graph neural network for large-scale quantum
  calculations.
\newblock \emph{arXiv preprint arXiv:2103.01436}, 2021{\natexlab{b}}.

\bibitem[Jing et~al.(2022)Jing, Corso, Barzilay, and
  Jaakkola]{jing2022torsional}
Bowen Jing, Gabriele Corso, Regina Barzilay, and Tommi~S. Jaakkola.
\newblock Torsional diffusion for molecular conformer generation.
\newblock In \emph{ICLR Workshop on Deep Generative Models for Highly
  Structured Data}, 2022.
\newblock URL \url{https://openreview.net/forum?id=SBgNnnVuwbc}.

\bibitem[Jumper et~al.(2021)Jumper, Evans, Pritzel, Green, Figurnov,
  Ronneberger, Tunyasuvunakool, Bates, {\v{Z}}{\'\i}dek, Potapenko,
  et~al.]{jumper2021highly}
John Jumper, Richard Evans, Alexander Pritzel, Tim Green, Michael Figurnov,
  Olaf Ronneberger, Kathryn Tunyasuvunakool, Russ Bates, Augustin
  {\v{Z}}{\'\i}dek, Anna Potapenko, et~al.
\newblock Highly accurate protein structure prediction with {AlphaFold}.
\newblock \emph{Nature}, 596\penalty0 (7873):\penalty0 583--589, 2021.

\bibitem[Kingma and Ba(2015)]{kingma2014adam}
Diederik Kingma and Jimmy Ba.
\newblock Adam: A method for stochastic optimization.
\newblock \emph{The International Conference on Learning Representations},
  2015.

\bibitem[Klicpera et~al.(2020)Klicpera, Giri, Margraf, and
  G{\"u}nnemann]{klicpera_dimenetpp_2020}
Johannes Klicpera, Shankari Giri, Johannes~T Margraf, and Stephan
  G{\"u}nnemann.
\newblock Fast and uncertainty-aware directional message passing for
  non-equilibrium molecules.
\newblock In \emph{NeurIPS-W}, 2020.

\bibitem[Kochkov et~al.(2021)Kochkov, Pfaff, Sanchez-Gonzalez, Battaglia, and
  Clark]{kochkov2021learning}
Dmitrii Kochkov, Tobias Pfaff, Alvaro Sanchez-Gonzalez, Peter Battaglia, and
  Bryan~K Clark.
\newblock Learning ground states of quantum hamiltonians with graph networks.
\newblock \emph{arXiv preprint arXiv:2110.06390}, 2021.

\bibitem[Liu et~al.(2021)Liu, Luo, Wang, Xie, Yuan, Gui, Yu, Xu, Zhang, Liu,
  Yan, Liu, Fu, Oztekin, Zhang, and Ji]{liu2021dig}
Meng Liu, Youzhi Luo, Limei Wang, Yaochen Xie, Hao Yuan, Shurui Gui, Haiyang
  Yu, Zhao Xu, Jingtun Zhang, Yi~Liu, Keqiang Yan, Haoran Liu, Cong Fu, Bora~M
  Oztekin, Xuan Zhang, and Shuiwang Ji.
\newblock {DIG}: A turnkey library for diving into graph deep learning
  research.
\newblock \emph{Journal of Machine Learning Research}, 22\penalty0
  (240):\penalty0 1--9, 2021.
\newblock URL \url{http://jmlr.org/papers/v22/21-0343.html}.

\bibitem[Liu et~al.(2020)Liu, Yuan, Cai, and Ji]{liu2020deep}
Yi~Liu, Hao Yuan, Lei Cai, and Shuiwang Ji.
\newblock Deep learning of high-order interactions for protein interface
  prediction.
\newblock In \emph{Proceedings of the 26th ACM SIGKDD International Conference
  on Knowledge Discovery \& Data Mining}, pages 679--687, 2020.

\bibitem[Liu et~al.(2022)Liu, Wang, Liu, Lin, Zhang, Oztekin, and
  Ji]{liu2022spherical}
Yi~Liu, Limei Wang, Meng Liu, Yuchao Lin, Xuan Zhang, Bora Oztekin, and
  Shuiwang Ji.
\newblock Spherical message passing for {3D} molecular graphs.
\newblock In \emph{International Conference on Learning Representations}, 2022.
\newblock URL \url{https://openreview.net/forum?id=givsRXsOt9r}.

\bibitem[Lu et~al.(2019)Lu, Liu, Wang, Huang, Lin, and He]{lu2019molecular}
Chengqiang Lu, Qi~Liu, Chao Wang, Zhenya Huang, Peize Lin, and Lixin He.
\newblock Molecular property prediction: A multilevel quantum interactions
  modeling perspective.
\newblock In \emph{Proceedings of the AAAI Conference on Artificial
  Intelligence}, volume~33, pages 1052--1060, 2019.

\bibitem[Maron et~al.(2019)Maron, Ben-Hamu, Serviansky, and
  Lipman]{maron2019provably}
Haggai Maron, Heli Ben-Hamu, Hadar Serviansky, and Yaron Lipman.
\newblock Provably powerful graph networks.
\newblock \emph{Advances in Neural Information Processing Systems},
  32:\penalty0 2153--2164, 2019.

\bibitem[Morehead et~al.(2022)Morehead, Chen, and Cheng]{morehead2021geometric}
Alex Morehead, Chen Chen, and Jianlin Cheng.
\newblock Geometric transformers for protein interface contact prediction.
\newblock In \emph{International Conference on Learning Representations}, 2022.
\newblock URL \url{https://openreview.net/forum?id=CS4463zx6Hi}.

\bibitem[Morris et~al.(2019)Morris, Ritzert, Fey, Hamilton, Lenssen, Rattan,
  and Grohe]{morris2019weisfeiler}
Christopher Morris, Martin Ritzert, Matthias Fey, William~L Hamilton, Jan~Eric
  Lenssen, Gaurav Rattan, and Martin Grohe.
\newblock Weisfeiler and leman go neural: Higher-order graph neural networks.
\newblock In \emph{Proceedings of the AAAI Conference on Artificial
  Intelligence}, volume~33, pages 4602--4609, 2019.

\bibitem[Nakata and Shimazaki(2017)]{nakata2017pubchemqc}
Maho Nakata and Tomomi Shimazaki.
\newblock {PubChemQC} project: a large-scale first-principles electronic
  structure database for data-driven chemistry.
\newblock \emph{Journal of chemical information and modeling}, 57\penalty0
  (6):\penalty0 1300--1308, 2017.

\bibitem[Ramakrishnan et~al.(2014)Ramakrishnan, Dral, Rupp, and
  Von~Lilienfeld]{ramakrishnan2014quantum}
Raghunathan Ramakrishnan, Pavlo~O Dral, Matthias Rupp, and O~Anatole
  Von~Lilienfeld.
\newblock Quantum chemistry structures and properties of 134 kilo molecules.
\newblock \emph{Scientific data}, 1\penalty0 (1):\penalty0 1--7, 2014.

\bibitem[Satorras et~al.(2021)Satorras, Hoogeboom, and Welling]{satorras2021n}
V{\i}ctor~Garcia Satorras, Emiel Hoogeboom, and Max Welling.
\newblock E (n) equivariant graph neural networks.
\newblock In \emph{International conference on machine learning}, pages
  9323--9332. PMLR, 2021.

\bibitem[Schmidhuber(2015)]{schmidhuber2015deep}
J{\"u}rgen Schmidhuber.
\newblock Deep learning in neural networks: An overview.
\newblock \emph{Neural Networks}, 61:\penalty0 85--117, 2015.

\bibitem[Sch{\"u}tt et~al.(2017)Sch{\"u}tt, Kindermans, Felix, Chmiela,
  Tkatchenko, and M{\"u}ller]{schutt2017schnet}
Kristof Sch{\"u}tt, Pieter-Jan Kindermans, Huziel Enoc~Sauceda Felix, Stefan
  Chmiela, Alexandre Tkatchenko, and Klaus-Robert M{\"u}ller.
\newblock {SchNet}: A continuous-filter convolutional neural network for
  modeling quantum interactions.
\newblock In \emph{Advances in Neural Information Processing Systems}, pages
  991--1001, 2017.

\bibitem[Sch{\"u}tt et~al.(2021)Sch{\"u}tt, Unke, and
  Gastegger]{schutt2021equivariant}
Kristof Sch{\"u}tt, Oliver Unke, and Michael Gastegger.
\newblock Equivariant message passing for the prediction of tensorial
  properties and molecular spectra.
\newblock In \emph{International Conference on Machine Learning}, pages
  9377--9388. PMLR, 2021.

\bibitem[Shuaibi et~al.(2021)Shuaibi, Kolluru, Das, Grover, Sriram, Ulissi, and
  Zitnick]{shuaibi2021rotation}
Muhammed Shuaibi, Adeesh Kolluru, Abhishek Das, Aditya Grover, Anuroop Sriram,
  Zachary Ulissi, and C~Lawrence Zitnick.
\newblock Rotation invariant graph neural networks using spin convolutions.
\newblock \emph{arXiv preprint arXiv:2106.09575}, 2021.

\bibitem[Simm et~al.(2020)Simm, Pinsler, and
  {Hern{\'a}ndez-Lobato}]{Simm2020Reinforcement}
Gregor N.~C. Simm, Robert Pinsler, and Jos{\'e}~Miguel {Hern{\'a}ndez-Lobato}.
\newblock Reinforcement learning for molecular design guided by quantum
  mechanics.
\newblock In Hal~Daum{\'e} III and Aarti Singh, editors, \emph{Proceedings of
  the 37th International Conference on Machine Learning}, volume 119 of
  \emph{Proceedings of Machine Learning Research}, pages 8959--8969. {PMLR},
  2020.
\newblock URL \url{http://proceedings.mlr.press/v119/simm20b.html}.

\bibitem[Simm et~al.(2021)Simm, Pinsler, Cs{\'a}nyi, and
  Hern{\'a}ndez-Lobato]{Simm2021SymmetryAware}
Gregor N.~C. Simm, Robert Pinsler, G{\'a}bor Cs{\'a}nyi, and Jos{\'e}~Miguel
  Hern{\'a}ndez-Lobato.
\newblock Symmetry-aware actor-critic for {3D} molecular design.
\newblock In \emph{International Conference on Learning Representations}, 2021.
\newblock URL \url{https://openreview.net/forum?id=jEYKjPE1xYN}.

\bibitem[Stokes et~al.(2020)Stokes, Yang, Swanson, Jin, Cubillos-Ruiz, Donghia,
  MacNair, French, Carfrae, Bloom-Ackermann, et~al.]{stokes2020deep}
Jonathan~M Stokes, Kevin Yang, Kyle Swanson, Wengong Jin, Andres Cubillos-Ruiz,
  Nina~M Donghia, Craig~R MacNair, Shawn French, Lindsey~A Carfrae, Zohar
  Bloom-Ackermann, et~al.
\newblock A deep learning approach to antibiotic discovery.
\newblock \emph{Cell}, 180\penalty0 (4):\penalty0 688--702, 2020.

\bibitem[Thomas et~al.(2018)Thomas, Smidt, Kearnes, Yang, Li, Kohlhoff, and
  Riley]{thomas2018tensor}
Nathaniel Thomas, Tess Smidt, Steven Kearnes, Lusann Yang, Li~Li, Kai Kohlhoff,
  and Patrick Riley.
\newblock Tensor field networks: Rotation-and translation-equivariant neural
  networks for {3D} point clouds.
\newblock \emph{arXiv preprint arXiv:1802.08219}, 2018.

\bibitem[Unke and Meuwly(2019)]{unke2019physnet}
Oliver~T Unke and Markus Meuwly.
\newblock {PhysNet}: A neural network for predicting energies, forces, dipole
  moments, and partial charges.
\newblock \emph{Journal of chemical theory and computation}, 15\penalty0
  (6):\penalty0 3678--3693, 2019.

\bibitem[Vignac et~al.(2020)Vignac, Loukas, and Frossard]{vignac2020building}
Clement Vignac, Andreas Loukas, and Pascal Frossard.
\newblock Building powerful and equivariant graph neural networks with
  structural message-passing.
\newblock \emph{Advances in Neural Information Processing Systems},
  33:\penalty0 14143--14155, 2020.

\bibitem[Wang et~al.(2022)Wang, Liu, Luo, Xu, Xie, Wang, Cai, Qi, Yuan, Yang,
  et~al.]{wang2022advanced}
Zhengyang Wang, Meng Liu, Youzhi Luo, Zhao Xu, Yaochen Xie, Limei Wang, Lei
  Cai, Qi~Qi, Zhuoning Yuan, Tianbao Yang, et~al.
\newblock Advanced graph and sequence neural networks for molecular property
  prediction and drug discovery.
\newblock \emph{Bioinformatics}, 38\penalty0 (9):\penalty0 2579--2586, 2022.

\bibitem[Wu et~al.(2018)Wu, Ramsundar, Feinberg, Gomes, Geniesse, Pappu,
  Leswing, and Pande]{wu2018moleculenet}
Zhenqin Wu, Bharath Ramsundar, Evan~N Feinberg, Joseph Gomes, Caleb Geniesse,
  Aneesh~S Pappu, Karl Leswing, and Vijay Pande.
\newblock {MoleculeNet}: a benchmark for molecular machine learning.
\newblock \emph{Chemical science}, 9\penalty0 (2):\penalty0 513--530, 2018.

\bibitem[Xie and Grossman(2018)]{xie2018crystal}
Tian Xie and Jeffrey~C Grossman.
\newblock Crystal graph convolutional neural networks for an accurate and
  interpretable prediction of material properties.
\newblock \emph{Physical review letters}, 120\penalty0 (14):\penalty0 145301,
  2018.

\bibitem[Xu et~al.(2019)Xu, Hu, Leskovec, and Jegelka]{xu2018powerful}
Keyulu Xu, Weihua Hu, Jure Leskovec, and Stefanie Jegelka.
\newblock How powerful are graph neural networks?
\newblock In \emph{International Conference on Learning Representations}, 2019.

\bibitem[Xu et~al.(2021)Xu, Luo, Zhang, Xu, Xie, Liu, Dickerson, Deng, Nakata,
  and Ji]{xu2021molecule3d}
Zhao Xu, Youzhi Luo, Xuan Zhang, Xinyi Xu, Yaochen Xie, Meng Liu, Kaleb
  Dickerson, Cheng Deng, Maho Nakata, and Shuiwang Ji.
\newblock {Molecule3D}: A benchmark for predicting {3D} geometries from
  molecular graphs.
\newblock \emph{arXiv preprint arXiv:2110.01717}, 2021.

\bibitem[Yang et~al.(2019)Yang, Swanson, Jin, Coley, Eiden, Gao, Guzman-Perez,
  Hopper, Kelley, Mathea, et~al.]{yang2019analyzing}
Kevin Yang, Kyle Swanson, Wengong Jin, Connor Coley, Philipp Eiden, Hua Gao,
  Angel Guzman-Perez, Timothy Hopper, Brian Kelley, Miriam Mathea, et~al.
\newblock Analyzing learned molecular representations for property prediction.
\newblock \emph{Journal of chemical information and modeling}, 59\penalty0
  (8):\penalty0 3370--3388, 2019.

\bibitem[Ying et~al.(2021)Ying, Cai, Luo, Zheng, Ke, He, Shen, and
  Liu]{ying2021do}
Chengxuan Ying, Tianle Cai, Shengjie Luo, Shuxin Zheng, Guolin Ke, Di~He,
  Yanming Shen, and Tie-Yan Liu.
\newblock Do transformers really perform badly for graph representation?
\newblock In \emph{Thirty-Fifth Conference on Neural Information Processing
  Systems}, 2021.
\newblock URL \url{https://openreview.net/forum?id=OeWooOxFwDa}.

\bibitem[Zhang et~al.(2018)Zhang, Cui, Neumann, and Chen]{zhang2018end}
Muhan Zhang, Zhicheng Cui, Marion Neumann, and Yixin Chen.
\newblock An end-to-end deep learning architecture for graph classification.
\newblock In \emph{Proceedings of the AAAI Conference on Artificial
  Intelligence}, volume~32, 2018.

\end{thebibliography}
\bibliographystyle{plainnat} 

\newpage
\appendix

\section{Appendix}

\subsection{Algorithm for the Proposed Message Passing} \label{sec:alg}

\begin{algorithm}[H]
\caption{The Proposed Message Passing}\label{alg:alg}
\begin{algorithmic}[1]
\FOR{$i=1,..., n$}
\FOR{$j=1,..., k$}
\STATE Compute $d_{ij}=||\mathrm{\textbf{p}}_i-\mathrm{\textbf{p}}_j||_2$
\ENDFOR
\STATE Get reference nodes via
\vspace{-6pt}
$$f_i=\mathrm{argmin}_{k\in\mathcal{N}_i} (d_{ik}), s_i=\mathrm{argmin}_{k\in\mathcal{N}_i \backslash \{f_{i}\}} (d_{ik})$$
\vspace{-16pt}
\ENDFOR
\FOR{$i=1,..., n$}
\FOR{$j=1,..., k$}
\STATE Get reference nodes via
\vspace{-6pt}
$$f_{i\backslash j}= 
\begin{cases}
    f_i,& \text{if } f_i \neq j\\
    s_i,              & \text{otherwise}
\end{cases},
f_{j\backslash i}= 
\begin{cases}
    f_j,& \text{if } f_j \neq i\\
    s_j,              & \text{otherwise}
\end{cases}
$$
\vspace{-10pt}
\STATE Compute angles via
\vspace{-6pt}
\begin{equation*}
\begin{aligned}
    \theta_{ij} &= \mathrm{angle}_1(f_i,i,j), \\
    \phi_{ij} &= \mathrm{angle}_2(\mathrm{plane}_{f_i,i,s_i}, \mathrm{plane}_{f_i,i,j}), \\
    \tau_{ij} &= \mathrm{angle}_3(\mathrm{plane}_{f_{i\backslash j},i,j}, \mathrm{plane}_{i,j,f_{j\backslash i}}), \\
\end{aligned}
\end{equation*}
\vspace{-15pt}
\ENDFOR
\STATE Update node features via Eq.~\ref{eq:mp}
\ENDFOR
\end{algorithmic}
\end{algorithm}

As rigorously shown in Algorithm~\ref{alg:alg} , there are two nested loops in the message passing.
For each nested loop, the complexity for the outer loop is $O(n)$, and the complexity for the inner loop is $O(k)$.
Particularly, within the inner loop,
the operations are picking proper reference nodes, thus the complexity is simply $O(1)$. 
Overall, the total complexity of our message passing is $O(nk)$.
Importantly, the efficiency of our method is also demonstrated in Sec.~\ref{sec:exp} that ComeNet is 6-10 times faster than SphereNet.
Actually, the training time of ComENet is similar to SchNet whose complexity is also $O(nk)$.

\subsection{SE(3)} \label{SE(3)}
\textcolor{COLOR}{
SE(3) is the Special Euclidean group in 3 dimensions, including all rotations and translations in 3D. It is the set of $4\times 4$ real matrices of the form 
\begin{equation}
    \begin{bmatrix}
    Rot & \textbf{t} \\
    0 & 1 
    \end{bmatrix}
    =
    \begin{bmatrix}
    r_{11} & r_{12} & r_{13} & t_1\\
    r_{21} & r_{22} & r_{23} & t_2\\
    r_{31} & r_{32} & r_{33} & t_3\\
    0 &0 &0 &1
    \end{bmatrix},
\end{equation}
where $Rot\in \mathrm{SO(3)}$ and $\textbf{t}\in \mathbb{R}^{3}$. 
The SO(3) is the set of $3\times 3$ real matrices $Rot$ satisfying 
$Rot^TRot=I$ and $\det(Rot)=1$.
}

\subsection{Proofs of Lemma~\ref{lemma1}} \label{sec:proof}

\captionsetup[subfigure]{labelformat=empty}
\begin{figure*}[ht]
     \centering
     \subfloat[Case (1): $i$ has more than one neighboring nodes $c_{1,2,...}$]
     {\includegraphics[width=0.4\textwidth]{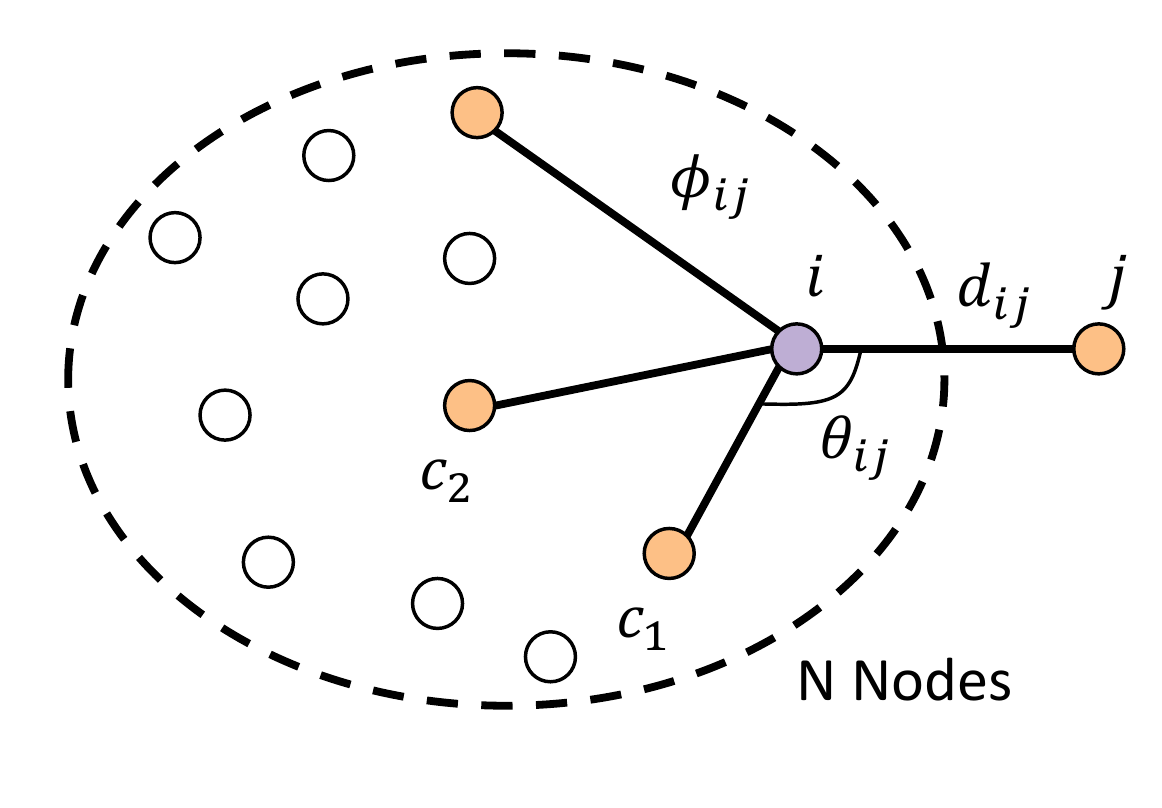}\label{fig:case1}}
     \qquad
     \subfloat[Case (2): $i$ only has one neighboring node $c_1$]
     {\includegraphics[width=0.4\textwidth]{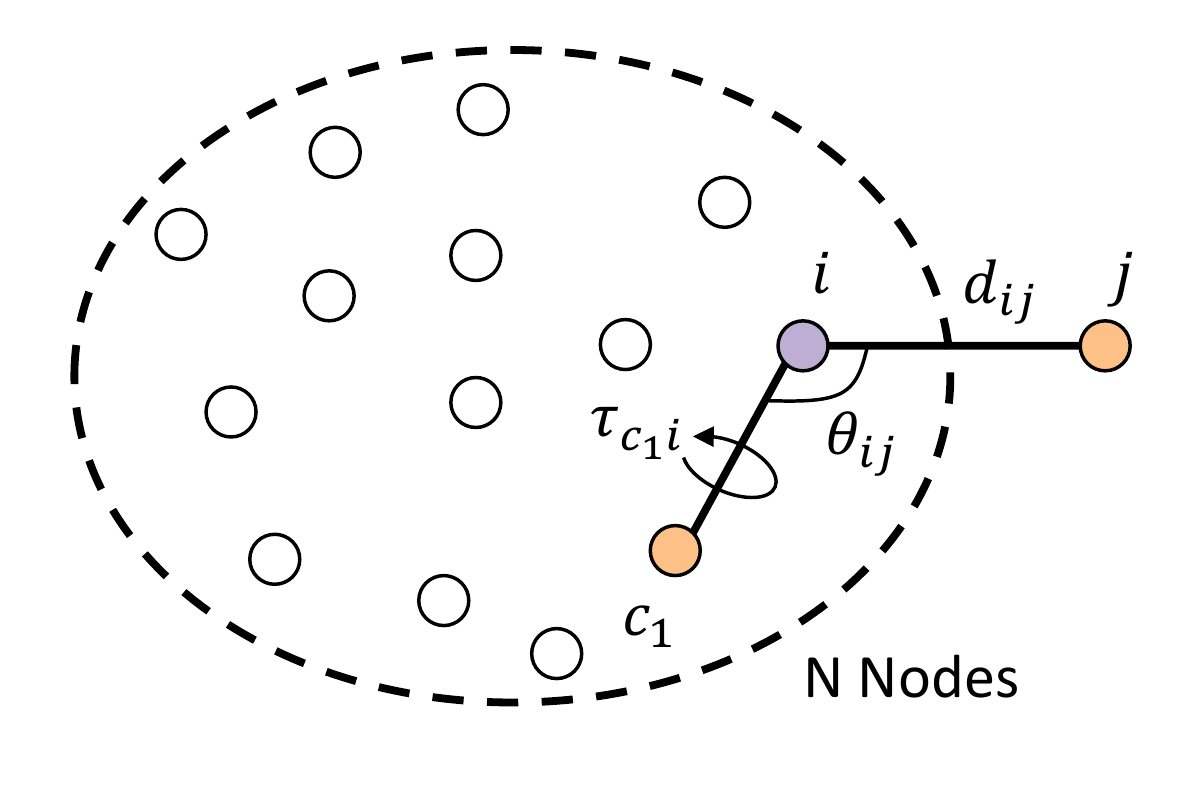}\label{fig:case2}}
    \caption{Two cases for proving Lemma~\ref{lemma1} when generalizing the size of a 3D graph from $k$ to $k+1$.}
    \label{fig:n_to_n+1}
    \vspace{-10pt}
\end{figure*}

\begin{proof}
We employ contradiction to prove it. Basically, there exist two cases when adding this new node $j$,
as illustrated in Fig.~\ref{fig:n_to_n+1} and described as following.
\emph{Case (1): $i$ has more than one neighboring nodes $c_{1,2,...}$;
Case (2): $i$ only has one neighboring node $c_1$.}

Generally, in case (1), we need to prove $\textbf{p}_{ij}$ is uniquely defined as $(d_{ij}, \theta_{ij}, \phi_{ij})$; in case (2), we need to prove $\textbf{p}_{ij}$ is uniquely defined as $(d_{ij}, \theta_{ij}, \tau_{c_1i})$.
We assume there exists another location for $j$ such that there is a different relative location vector
$\widetilde{\textbf{p}}_{ij}$, and this assumption leads to contradiction.
The proofs of both cases are provided below.
\end{proof}

\emph{Proof of case (1)}:
In case (1), we need to prove $\textbf{p}_{ij}$ is uniquely defined as $(d_{ij}, \theta_{ij}, \phi_{ij})$.
Based on the notations defined in Sec.~\ref{sec:local}, we have
\begin{equation}
\begin{aligned}
    \left<\textbf{p}_{ij}, \textbf{p}_{ij} \right> = d^2_{ij}, \\
    \left<\textbf{p}_{if_i}, \textbf{p}_{ij} \right> = d_{if_i}d_{ij}cos\theta_{ij}, \\
    \left<T(\textbf{p}_{is_i}), T(\textbf{p}_{ij}) \right> = \Vert T(\textbf{p}_{is_i})\Vert \Vert T(\textbf{p}_{ij})\Vert cos\phi_{ij}, \\
    \left<T(\textbf{p}_{is_i})\times T(\textbf{p}_{ij}),  \textbf{p}_{if_i} \right> =d_{if_i} \Vert T(\textbf{p}_{is_i})\Vert \Vert T(\textbf{p}_{ij})\Vert sin\phi_{ij}, \\
    \label{eq:case1_1}
\end{aligned}
\end{equation}
where $T$ denotes an operator of projection to 
the plane perpendicular to $\textbf{p}_{if_i}$.
Apparently, all the quantities on the right side are known based on 
Sec.~\ref{sec:local}. 
Assume the solution set
contains at least two different solutions $\textbf{p}_{ij}$
and $\widetilde{\textbf{p}}_{ij}$.
Eq.~\ref{eq:case1_1} can simply imply that
\begin{equation}
\begin{aligned}
    \left<\textbf{p}_{if_i}, \textbf{p}_{ij}- \widetilde{\textbf{p}}_{ij} \right> = 0, \\
    \left<T(\textbf{p}_{is_i}), \textbf{p}_{ij}- \widetilde{\textbf{p}}_{ij} \right> = 0, \\
    \left<T(\textbf{p}_{is_i}) \times (\textbf{p}_{ij}- \widetilde{\textbf{p}}_{ij}), \textbf{p}_{if_i}\right> = 0. \\
\label{eq:case1_2}
\end{aligned}
\end{equation}
Apparently, both $T(\textbf{p}_{is_i})$ and 
$\textbf{p}_{ij}- \widetilde{\textbf{p}}_{ij}$
are in the plane perpendicular to $\textbf{p}_{if_i}$.
Hence, $T(\textbf{p}_{is_i}) \times (\textbf{p}_{ij}- \widetilde{\textbf{p}}_{ij}) = \alpha \textbf{p}_{if_i}$
holds for some $\alpha\neq0$.
However, from Eq.~\ref{eq:case1_2}, we have
\begin{equation}
\begin{aligned}
    \left<T(\textbf{p}_{is_i}) \times (\textbf{p}_{ij}- \widetilde{\textbf{p}}_{ij}), \textbf{p}_{if_i}\right>  = \alpha\left<\textbf{p}_{if_i}, \textbf{p}_{if_i} \right> = 0. \\
\label{eq:case1_3}
\end{aligned}
\end{equation}
Since $\alpha\neq0$ and $\textbf{p}_{if_i}\neq\textbf{0}$, Eq.~\ref{eq:case1_3} causes a contradiction. Thus, the solution set contains a unique solution.

\emph{Proof of case (2)}:
In case (2), we need to prove $\textbf{p}_{ij}$ is uniquely defined by $(d_{ij}, \theta_{ij}, \tau_{c_1i})$.
Based on the notations defined in Sec.~\ref{sec:local} and Sec.~\ref{sec:global}, we have
\begin{equation}
\begin{aligned}
    \left<\textbf{p}_{ij}, \textbf{p}_{ij} \right> = d^2_{ij}, \\
    \left<\textbf{p}_{if_i}, \textbf{p}_{ij} \right> = d_{if_i}d_{ij}cos\theta_{ij}, \\
    \left<T(\textbf{p}_{f_if_{f_i\backslash i}}), T(\textbf{p}_{ij}) \right> = \Vert T(\textbf{p}_{f_if_{f_i\backslash i}})\Vert \Vert T(\textbf{p}_{ij})\Vert cos\tau_{c_1i}, \\
    \left<T(\textbf{p}_{f_if_{f_i\backslash i}}) \times T(\textbf{p}_{ij}),  \textbf{p}_{if_i} \right> =d_{if_i} \Vert T(\textbf{p}_{f_if_{f_i\backslash i}})\Vert \Vert T(\textbf{p}_{ij})\Vert sin\tau_{c_1i}, \\
    \label{eq:case2_1}
\end{aligned}
\end{equation}
where $T$ still denotes an operator of projection to 
the plane perpendicular to $\textbf{p}_{if_i}$.
Similarly, all the quantities on the right side are known based on 
Sec.~\ref{sec:local} and Sec.~\ref{sec:global}. 
Assume the solution set
contains at least two different solutions $\textbf{p}_{ij}$
and $\widetilde{\textbf{p}}_{ij}$.
Eq.~\ref{eq:case2_1} can simply imply that
\begin{equation}
\begin{aligned}
    \left<\textbf{p}_{if_i}, \textbf{p}_{ij}- \widetilde{\textbf{p}}_{ij} \right> = 0, \\
    \left<T(\textbf{p}_{f_if_{f_i\backslash i}}), \textbf{p}_{ij}- \widetilde{\textbf{p}}_{ij} \right> = 0, \\
    \left<T(\textbf{p}_{f_if_{f_i\backslash i}}) \times (\textbf{p}_{ij}- \widetilde{\textbf{p}}_{ij}), \textbf{p}_{if_i}\right> = 0. \\
\label{eq:case2_2}
\end{aligned}
\end{equation}

Both $T(\textbf{p}_{f_if_{f_i\backslash i}})$ and 
$\textbf{p}_{ij}- \widetilde{\textbf{p}}_{ij}$
are in the plane perpendicular to $\textbf{p}_{if_i}$.
Hence, $T(\textbf{p}_{f_if_{f_i\backslash i}}) \times (\textbf{p}_{ij}- \widetilde{\textbf{p}}_{ij}) = \alpha \textbf{p}_{if_i}$
holds for some $\alpha\neq0$.
However, from Eq.~\ref{eq:case2_2}, we have
\begin{equation}
\begin{aligned}
    \left<T(\textbf{p}_{f_if_{f_i\backslash i}}) \times (\textbf{p}_{ij}- \widetilde{\textbf{p}}_{ij}), \textbf{p}_{if_i}\right>  = \alpha\left<\textbf{p}_{if_i}, \textbf{p}_{if_i} \right> = 0. \\
\label{eq:case2_3}
\end{aligned}
\end{equation}
Since $\alpha\neq0$ and $\textbf{p}_{if_i}\neq\textbf{0}$, Eq.~\ref{eq:case2_3} causes a contradiction. Thus, the solution set contains a unique solution.


\subsection{Model architecture} \label{sec:model architecture}

\textbf{Interaction Layer} updates each node feature vector $\textbf{v}$ based on features of the neighboring nodes and the corresponding 3D information in $P$. 
Firstly, it converts 3D information in $P$ to a set of geometries
based on the proposed complete geometric transformation $\mathcal{T}$ and message passing scheme.
Since distance is the most important geometry, we also consider $d$ in global representation and split the output of $\mathcal{T}$ into two tuples $(d,\theta,\phi)$ and $(d,\tau)$,
for local and global representations, respectively.

Importantly, the tuples $(d,\theta,\phi)$ and $(d,\tau)$ cannot serve as immediate inputs to the network.
They need to be transformed into physically meaningful vectors based on quantum-based basis functions.
As in previous studies~\cite{hu2021forcenet,klicpera_dimenet_2020,liu2022spherical,klicpera2021gemnet}, we test different basis functions including MLP, Gaussian and sine functions, spherical Bessel functions, and spherical harmonics. We found spherical Bessel and spherical harmonics perform best.
Formally, the basis function for tuple $(d,\theta,\phi)$ is TBF $j_{\ell}\left(\frac{\beta_{\ell n}}{c}d\right)Y_{\ell}^{m}(\theta, \phi)$,
where $j_{\ell}(\cdot)$ is a spherical Bessel function of order $\ell$,
$Y_{\ell}^{m}$ is a spherical harmonic function of degree $m$ and order $\ell$,
$c$ is the cutoff,
and $\beta_{\ell n}$ is the $n$-th root of the Bessel function of order $\ell$. 
The basis function for tuple $(d,\tau)$ is SBF $j_{\ell}\left(\frac{\beta_{\ell n}}{c}d\right)Y_{\ell}^{0}(\tau)$.
These two basis functions are also used in SphereNet~\cite{liu2022spherical} and GemNet~\cite{klicpera2021gemnet}.

The two physically meaningful vectors from TBF and SBF are then imported into
a local convolution layer and a global convolution layer, respectively.
For both convolution layers, we use the GraphConv~\cite{morris2019weisfeiler} implemented in the PyTorch Geometric library~\cite{Fey/Lenssen/2019}.
The vectors from the basis functions are used as edge weights in the convolution layers.
The outputs of local and global convolution layers are concatenated to generate a new node feature vector. 
Then the concatenated vector is forwarded into several linear layers to generate the updated
feature vector $\textbf{v}^\prime$.

\subsection{Data Description} \label{sec:data}

\textbf{OC20.} The Open Catalyst 2020 (OC20) dataset~\cite{chanussot2021open} is a newly released dataset to model and discover catalysts.
Specifically, the goal is efficient DFT approximation of structure relaxation, 
which is a fundamental calculation in catalysis to determine a structure's activity and selectivity.
All the structures in the dataset contain a surface and an adsorbate, 
and the surface is defined by a unit cell that is periodic in all directions.
There are three tasks including 
Structure to Energy and Forces (S2EF),
Initial Structure to Relaxed Structure (IS2RS), 
and Initial Structure to Relaxed Energy (IS2RE).

In this work, we focus on Initial Structure to Relaxed Energy (IS2RE) task, which is the most common task in catalysis 
as the relaxed energies are often correlated with catalyst activity and selectivity. 
The dataset for IS2RE is originally split into training, validation, and test sets.
The training set contains 460,328 structures and the validation set has four splits including 
in-domain (ID), 
out-of-domain adsorbate (OOD Ads), 
out-of-domain catalyst (OOD Cat), 
and out-of-domain adsorbate and catalyst (OOD Both), 
with 24,733, 24,961, 24,738, 24,971 structures respectively.

\textbf{Molecule3D.} The Molecule3D dataset~\cite{xu2021molecule3d} is a newly proposed large-scale dataset, including around 4 million molecules with precise ground-state geometries derived from DFT. 
The dataset is collected from PubChemQC~\cite{nakata2017pubchemqc} and designed for predicting 3D geometries from molecular graphs~\cite{xu2021molecule3d} while 
we aim to learn representations and predict properties for molecules based on their geometries.
The dataset contains 3,899,647 molecules and is split into training, validation, and test sets via random and scaffold split with ratio 6:2:2. 
Random split ensures the training, validation, and test data are sampled from the same distribution while 
scaffold split leads to a distribution shift between training and test data.
The evaluation metric is the MAE between the predictions and the ground truth.

\textcolor{COLOR}{\textbf{QM9.} The QM9 dataset~\citep{ramakrishnan2014quantum} is a widely used dataset for predicting various properties of molecules. It includes geometric, energetic, electronic, and thermodynamic properties for 134k stable small organic molecules. The dataset is split into three sets, where the training set contains 110,000, the validation set contains 10,000, and the test set contains 10,831 molecules. The twelve properties are dipole moment ($\mu$), isotropic polarizability ($\alpha$), highest occupied molecular orbital energy ($\epsilon_\text{HOMO}$), lowest unoccupied molecular orbital energy ($\epsilon_\text{LUMO}$), gap between $\epsilon_\text{HOMO}$ and $\epsilon_\text{LUMO}$, electronic spatial extent ($\left< R^2 \right>$), zero point vibrational energy (ZPVE), internal energy at 0K ($U_0$), internal energy at 298.15K ($U$) , enthalpy at 298.15K ($H$), free energy at 298.15K ($G$) , and heat capavity at 298.15K ($c_\text{v}$). }

\subsection{Experimental Setup}   \label{sec:setup}
\textbf{ComENet.} The values/search space of model and training hyperparameters for ComENet on OC20, Molecule3D, and QM9 are provided in Table~\ref{tb:model_hyperpara} and Table~\ref{tb:training_hyperpara}. For Molecule3D and QM9, the optimal hyperparameters are chosen by the performance on the validation sets. For OC20, since the final comparison results are on validation set, we firstly use 10\% of the training data to choose optimal hyperparameters, then train our model on whole training data.
Specifically, we use a larger cutoff value for OC20 to generate graphs and larger hidden dimensions for OC20 and Molecule3D.
All models are trained on NVIDIA GeForce RTX 2080 Ti 11GB GPU for Molecule3D and QM9. 
For the OC20 dataset, we use NVIDIA RTX A6000 48GB GPU. Note that one experiment is only conducted on one GPU.

\textbf{Baselines for Molecule3D.} As Molecule3D is a newly proposed dataset, we run experiments and provide results for baseline methods including GIN-Virtual~\cite{hu2021ogb}, SchNet~\cite{schutt2017schnet}, DimeNet++\cite{klicpera_dimenetpp_2020} and SphereNet~\cite{liu2022spherical}. 
All the models are trained on one GPU (Nvidia GeForce RTX 2080 Ti 11GB). The model hyperparameters and training hyperparameters are listed in Table~\ref{tb:baseline_model_hyperpara} and Table~\ref{tb:baseline_training_hyperpara}. The optimal hyperparameters are chosen by the performance on validation set. Note that for DimeNet++\cite{klicpera_dimenetpp_2020} and SphereNet~\cite{liu2022spherical}, the maximum batch size is 32 due to the GPU memory limitation.

\begin{table*}[ht]
\begin{center}
\caption{Model hyperparameters for ComENet.}
\label{tb:model_hyperpara}
{\begin{tabular}{l|ccc}\toprule
&\multicolumn{3}{c}{Values/ search space} \\
\cmidrule(l{4pt}r{4pt}){2-4}
Model Hyperparameters & OC20 &Molecule3D &QM9\\
\midrule
Number of layers &4, 5, 6 &4, 5, 6, 8 &4, 5, 6 \\
Cutoff &6.0, 8.0 &5.0 &5.0 \\
Hidden dim &128, 256, 512 &128, 256, 512 &128, 256 \\
Hidden dim in Self-Atom layer &128, 256, 512 &128, 256, 512 &128, 256 \\
Number of layers in Self-Atom layer &2, 3, 4 &2, 3, 4 &2, 3, 4 \\
Number of layers of the MLP in Interaction layer &2, 3, 4 &2, 3, 4 &2, 3, 4 \\
Distance embedding dim & 6, 12 &6, 12 &6, 12 \\
Angle embedding dim & 3, 6 &3, 6 &3, 6 \\
\bottomrule
\end{tabular}}
\end{center}
\end{table*}

\begin{table}[ht] 
\begin{center}
\caption{Training hyperparameters for ComENet.}
\label{tb:training_hyperpara}
{\begin{tabular}{l|ccc}\toprule
&\multicolumn{3}{c}{Values/ search space} \\
\cmidrule(l{4pt}r{4pt}){2-4}
Training hyperparameters  &OC20 &Molecule3D &QM9 \\
\midrule
Epochs &20 &300 &1000 \\
Batch size &64, 128 &128, 256 &32, 64, 128 \\
Learning rate 1e-3, 5e-4, 2e-4 & 1e-3, 5e-4, 2e-4 & 1e-3, 5e-4, 2e-4 &1e-3, 5e-4, 2e-4 \\
Learning rate decay factor &0.4, 0.5, 0.6 &0.4, 0.5, 0.6 &0.4, 0.5, 0.6 \\
Learning rate decay epochs &Milestone [4,7,10,12] &20, 30, 50 &100, 200 \\
Warmup epochs &2 &-- &-- \\
Warmup factor &0.2 &-- &-- \\
\bottomrule
\end{tabular}}
\end{center}
\end{table}

\begin{table*}[ht]
\begin{center}
\caption{Model hyperparameters for baseline methods on Molecule3D.}
\label{tb:baseline_model_hyperpara}
{\begin{tabular}{l|cccc}\toprule
&\multicolumn{4}{c}{Values/ search space} \\
\cmidrule(l{4pt}r{4pt}){2-5}
Model Hyperparameters & GIN-Virtual &SchNet &DimeNet++ & SphereNet\\
\midrule
Number of layers &4, 6, 8 &4, 6, 8 &3, 4, 5, 6 &3, 4, 5, 6 \\
Cutoff &- &10.0 &6.0 &6.0 \\
Hidden dim &600 &256 &128 &128 \\
\bottomrule
\end{tabular}}
\end{center}
\end{table*}

\begin{table}[!ht] 
\begin{center}
\caption{Training hyperparameters for baseline methods on Molecule3D.}
\label{tb:baseline_training_hyperpara}
{\begin{tabular}{l|cccc}\toprule
&\multicolumn{4}{c}{Values/ search space} \\
\cmidrule(l{4pt}r{4pt}){2-5}
Training hyperparameters  &GIN-Virtual &SchNet &DimeNet++ &SphereNet \\
\midrule
Epochs &300 &300 &300 &300 \\
Batch size &64, 128, 256 &64, 128, 256 &16, 32 &16, 32 \\
Learning rate & 1e-3, 5e-4, 2e-4 & 1e-3, 5e-4, 2e-4 &1e-3, 5e-4, 2e-4 &1e-3, 5e-4, 2e-4 \\
Learning rate decay factor &0.4, 0.5, 0.6 &0.4, 0.5, 0.6 &0.4, 0.5, 0.6 &0.4, 0.5, 0.6 \\
Learning rate decay epochs &20, 30, 50 &20, 30, 50 &20, 30, 50 &20, 30, 50 \\
\bottomrule
\end{tabular}}
\end{center}
\end{table}

\end{document}